%% file: neurips_2026.tex
\documentclass{article}

\usepackage[main, final]{neurips_2026}

\usepackage{xurl}

\usepackage[utf8]{inputenc} %
\usepackage[T1]{fontenc}    %
\usepackage{hyperref}       %
\usepackage{url}            %
\usepackage{booktabs}       %
\usepackage{amsfonts}       %
\usepackage{nicefrac}       %
\usepackage{microtype}      %
\usepackage{xcolor}         %

\usepackage{tcolorbox}
\usepackage{cancel}
\usepackage{physics}
\usepackage{tabularx} %
\usepackage{placeins}
\usepackage{wrapfig}
\usepackage{subcaption}
\usepackage{enumitem}

\usepackage{amsthm}
\usepackage{aliascnt}
\usepackage[capitalize,noabbrev]{cleveref}

\AddToHook{cmd/appendix/before}{%
    \crefalias{section}{appendix}%
    \crefalias{subsection}{appendix}
    \crefalias{subsubsection}{appendix}
}

\usepackage{amsmath}
\usepackage{amssymb}
\usepackage{mathtools}

\usepackage{algorithm}
\usepackage{algorithmic}

\usepackage{acronym}
\newacro{LT}{Looped Transformer}

\theoremstyle{plain}

\newaliascnt{proposition}{theorem}
\newtheorem{proposition}[proposition]{Proposition}
\aliascntresetthe{proposition}

\theoremstyle{definition}

\theoremstyle{remark}

\crefname{proposition}{Prop.}{Props.}
\Crefname{proposition}{Prop.}{Props.}

\usepackage{etoc}

\usepackage[textsize=tiny]{todonotes}
\Crefname{equation}{Eq.}{Eqs.}
\crefname{equation}{Eq.}{Eqs.}

\Crefname{figure}{Fig.}{Figs.}
\crefname{figure}{Fig.}{Figs.}

\Crefname{section}{Sec.}{Secs.}
\crefname{section}{Sec.}{Secs.}

\Crefname{appendix}{App.}{Apps.}
\crefname{appendix}{App.}{Apps.}

\definecolor{darkmagenta}{rgb}{0.56, 0.0, 1.0} 
\usepackage{multirow}
\usepackage{inconsolata}
\usepackage{changepage} %
\usepackage{newtxmath}
\hypersetup{colorlinks,linkcolor={orange},citecolor={darkmagenta},urlcolor={orange}} 
\hypersetup{colorlinks,linkcolor={blue},citecolor={darkmagenta},urlcolor={blue}}

\input{math_commands.tex}

\definecolor{mplwinterblue}{RGB}{0, 0, 219}
\definecolor{mplwintergreen}{RGB}{0, 219, 96}
\definecolor{mplsummergreen}{RGB}{0, 96, 72}
\definecolor{mplsummeryellow}{RGB}{219, 219, 91}

\title{A Mechanistic Analysis of Looped Reasoning Language Models}

\author{%
  Hugh Blayney \\
  University of Oxford\\
  \texttt{hugh@robots.ox.ac.uk} \\
  \And
  Álvaro Arroyo \\
  University of Oxford\\
  \texttt{alvaro.arroyo@eng.ox.ac.uk} \\
  \And
  Johan Obando-Ceron \\
  Mila – Quebec AI Institute \\
  Universit\'e de Montr\'eal \\
  \And
  Pablo Samuel Castro \\
  Mila – Quebec AI Institute \\
  Universit\'e de Montr\'eal \\
  \And
  Aaron Courville \\
  Mila – Quebec AI Institute \\
  Universit\'e de Montr\'eal \\
  \And
  Michael Bronstein \\
  University of Oxford \\
  AITHYRA \\
  \And
  Xiaowen Dong \\
  University of Oxford\\
}

\makeatletter
\def\addcontentsline#1#2#3{%
\addtocontents{#1}{\protect\contentsline{#2}{#3}{\thepage}{\@currentHref}{}}}
\makeatother
\begin{document}
\etocdepthtag.toc{mtoc}
\etocsettagdepth{mtoc}{none}
\etocsettagdepth{appendix}{none}

\newbox\dummybox
\setbox\dummybox\vbox{%
  \tableofcontents 
}

\maketitle

\begin{abstract}
  \input{sections/abstract}
\end{abstract}

\section{Introduction}\label{sec:intro}
\input{sections/introduction}

\section{Preliminaries}\label{sec:prelim}
\input{sections/preliminaries}

\section{Latent Fixed Points in Looped LLMs}\label{sec:fixed_points}
\input{sections/fixed_point}

\section{Stages of Inference in Looped Models Mirror Feedforward Computation}\label{sec:stages_of_inference}
\input{sections/stages_of_inference}

\section{Related Work}\label{sec:related_work}
\input{sections/related_work}

\section{Discussion}\label{sec:conclusion}
\input{sections/conclusion}

\begin{ack}
HB acknowledges funding support from the EPSRC Centre for Doctoral Training in Autonomous Intelligent Machines and Systems No. EP/S024050/1. MB is partially supported by the EPSRC Turing AI World-Leading Research Fellowship No. EP/X040062/1 and EPSRC AI Hub No. EP/Y028872/1. AITHYRA is supported by the Austrian Academy of Sciences and the not-for-profit Boehringer Ingelheim Stiftung.

HB would like to thank Enrique Queipo-de-Llano for many highly useful discussions particularly around Mixing Score and ColSum Concentration, and Jonas Geiping for discussions around the existence of Orbit and Slider limiting behavior. The authors would like to thank all reviewers of this manuscript for their constructive comments.
\end{ack}

\bibliography{looped_llms}
\bibliographystyle{icml2025}

\newpage
\appendix
\onecolumn

\etocdepthtag.toc{appendix}
\etocsettocstyle{\section*{Appendix Contents}}{} %
\setcounter{tocdepth}{3}
\etocsettagdepth{mtoc}{none}           %
\etocsettagdepth{appendix}{subsection}  %
\tableofcontents                        %

\input{sections/appendix}

\FloatBarrier
\newpage
\input{checklist.tex}

\end{document}

%% file: math_commands.tex
\usepackage{amsmath,amsfonts,bm}

\def\eqref#1{equation~\ref{#1}}

\def\1{\bm{1}}

\def\ve{{\bm{e}}}

\def\vh{{\bm{h}}}

\def\vx{{\bm{x}}}

\def\mA{{\bm{A}}}
\def\mB{{\bm{B}}}

\def\mE{{\bm{E}}}

\def\mM{{\bm{M}}}

\def\mW{{\bm{W}}}
\def\mX{{\bm{X}}}
\def\mY{{\bm{Y}}}
\def\mZ{{\bm{Z}}}

\DeclareMathAlphabet{\mathsfit}{\encodingdefault}{\sfdefault}{m}{sl}
\SetMathAlphabet{\mathsfit}{bold}{\encodingdefault}{\sfdefault}{bx}{n}

\newcommand{\softmax}{\mathrm{softmax}}

\DeclareMathOperator{\attn}{\texttt{Attn}}
\DeclareMathOperator{\block}{B}
\DeclareMathOperator{\stack}{S}
\newenvironment{sketch}{%
  \proof}{\endproof}

\newcommand{\ouro}[1]{Ouro}
\newcommand{\raven}[1]{Huginn-0125}
\newcommand{\rllama}[1]{Retrofitted Llama}
\newcommand{\rolmo}[1]{Retrofitted OLMo-2}
\newcommand{\rtllama}[1]{Retrofitted TinyLlama}
\newcommand{\parcae}[1]{Parcae}

\newcommand{\mainmixing}[1]{Mixing Score}

%% file: sections/abstract.tex
Reasoning has become a central capability in large language models. Recent research has shown that reasoning performance can be improved by looping an LLM’s layers in the latent dimension, resulting in \textit{looped reasoning language models}. Despite promising results, few works have investigated how the internal dynamics and computation of looped models differ from that of standard feedforward models. In this paper, we conduct a mechanistic analysis of latent states and layer behavior in looped language models. We analyze cyclic recurrence and show that for many of the studied models each layer in the cycle converges to a distinct fixed point; consequently, the recurrent block follows a consistent cyclic trajectory in the latent space. We provide evidence that as these fixed points are reached, attention-head behavior stabilizes, leading to constant behavior -- \emph{stages of inference} -- across recurrences. Empirically, we discover that recurrent blocks tend to organize their computation such that it closely mirrors that of feedforward models, repeating this behavior in depth with each iteration. We study how architectural choices influence the emergence and stability of these cyclic fixed points and stages of inference, providing practical guidance for looped language model design. {Code implementation can be found \href{https://github.com/HughBlayney/looped-mechanistic-analysis}{here}.}

%% file: sections/introduction.tex
\begin{wrapfigure}{r}{0.5\textwidth}
    \vspace{-0.7cm}
    \centering
    \includegraphics[width=0.9\linewidth]{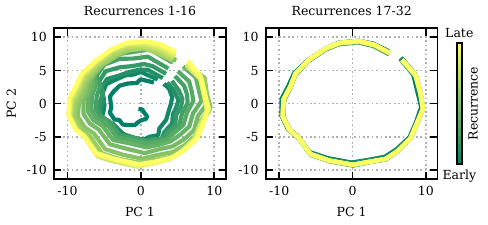}
    \vspace{-0.2cm}
    \caption{\textbf{Latent states after each block in a looped model frequently tend towards \emph{separate} fixed points}: a consistent trajectory in latent space. Illustrative; from a randomly initialized model.}
    \vspace{-0.3cm}
    \label{fig:random_trajectory}
\end{wrapfigure}

The vast majority of current LLMs are based on the Transformer architecture \citep{vaswani2017attention}, which comprises a sequence of blocks traversed in a feedforward manner to predict the next token. As the capability of these models increased, attention turned to eliciting \textit{reasoning} capabilities in LLMs by increasing test-time computation, commonly through chain-of-thought (CoT) prompting \citep{wei2022chain} or reinforcement-learning-based fine-tuning, notably DeepSeek-R1 \citep{guo2025deepseek}. More recently, research has explored building reasoning capabilities directly into the model architecture via \textit{recurrent looping} \citep{wang2025hierarchical, jolicoeur2025less}, where additional test-time compute is spent by taking more recurrent steps, echoing early designs in this direction \citep{graves2016adaptive}. Despite growing empirical success, the mechanisms underlying these models remain poorly understood, as well as their benefits and limitations when compared to feedforward computation.

In this paper, we study the dynamics of looped LLMs and how this impacts their organization of computation in depth. To this end, we study the macroscopic behavior of attention patterns, and the functional utility of latent states throughout the model depth via their performance on downstream tasks: this draws on existing research into \textit{stages of inference} \citep{lad2024remarkable, queipo2025attention} and related work \citep{skean2025layer,csordas2025do,ali2025entropy}. We show that looped LLMs often reach a fixed point with repeated looping, and we establish that stable recurrence necessitates the emergence of a \textit{constant cyclic trajectory} in latent space and consistent, repeated computational cycles. Remarkably, we find that the macroscopic computational behavior of looped LLMs frequently mirrors that of comparable feedforward models, but repeating within each recurrent block. Finally, we conduct a controlled experiment whereby we isolate the \emph{stability} property of a looped LLM and demonstrate that stable latent dynamics result in predictable test-time extrapolation and consistent computational patterns. These findings clarify both the relationship between looped and feedforward LLMs, and the role of stability in looped LLM computation, facilitating further development and the transfer of improvements between these fields.

\begin{tcolorbox}[colback=blue!5!white,colframe=blue!75!black,title=Key insights:]
\paragraph{Stability and Cyclic Trajectories:} We establish that many looped models converge to \emph{constant cyclic trajectories} in latent space. We show that this stability induces fixed, layer-specific attention patterns (\Cref{sec:fixed_points}) and identify the architectural drivers -- specifically input injection and normalization -- that facilitate this convergence (\Cref{subsec:norm_and_ii}).
\vspace{0.1cm}
\paragraph{Recurrent Stages of Inference:} We find that functional computational stages are repeated within each recurrence, and these are organised such that each loop mirrors the behavior of the entire depth of comparable feedforward models (\Cref{sec:stages_of_inference}). We provide preliminary evidence that stability of latent dynamics is a prerequisite for consistent extrapolation of these stages; unstable models fail to extrapolate to recurrences beyond those seen in training, while their functional stages diverge and drift across iterations (\Cref{sec:stability}).
\end{tcolorbox}

%% file: sections/preliminaries.tex
\subsection{\acp{LT}}

In this section we introduce \acp{LT} and define the notation that we will use throughout the paper. We represent the input sequence to our transformer of length $T$ and dimension $D$ as $\mX \in \mathbb{R}^{T \times D}$. Following the notation of \citet{yudin2025pay} we define $\attn$ as a map $f: \mathbb{R}^{T \times D} \to \mathbb{R}^{T \times D}$:
\begin{equation}
    \attn(\mX) = \softmax{\left ( \frac{\mX \mW_Q \mW_K^{\top} \mX^{\top}}{\sqrt{d}} \right )} \mX\mW_V = \softmax{\left ( A(\mX) \right )} \mX\mW_V, \label{eq:attn_in_a}
\end{equation}
where $\mW_Q, \mW_K, \mW_V \in \mathbb{R}^{D \times d}$ are projection matrices and $A$ is defined for convenience. A transformer block typically comprises an attention mechanism and a position-wise MLP as follows:
\begin{equation}
    \hat{\mX} = n_2 \left ( \mX + \attn(n_1(\mX)) \right ), \qquad {\mX}' = n_4 \left ( \hat{\mX} + \text{MLP}(n_3(\hat{\mX})) \right ),
\end{equation}
where $n_1, n_2, n_3, n_4$ are each optional norms, borrowing from the notation of \citet{geiping2025scaling}.
We denote the action of a Transformer block $\block: \mathbb{R}^{T \times D} \to \mathbb{R}^{T \times D}$ as $\mX' = \block(\mX)$, and refer to the intermediate hidden-state matrices $\mX$ between blocks as the \textbf{residual stream}.

\begin{figure}[htp]
    \centering
    \includegraphics[width=\linewidth]{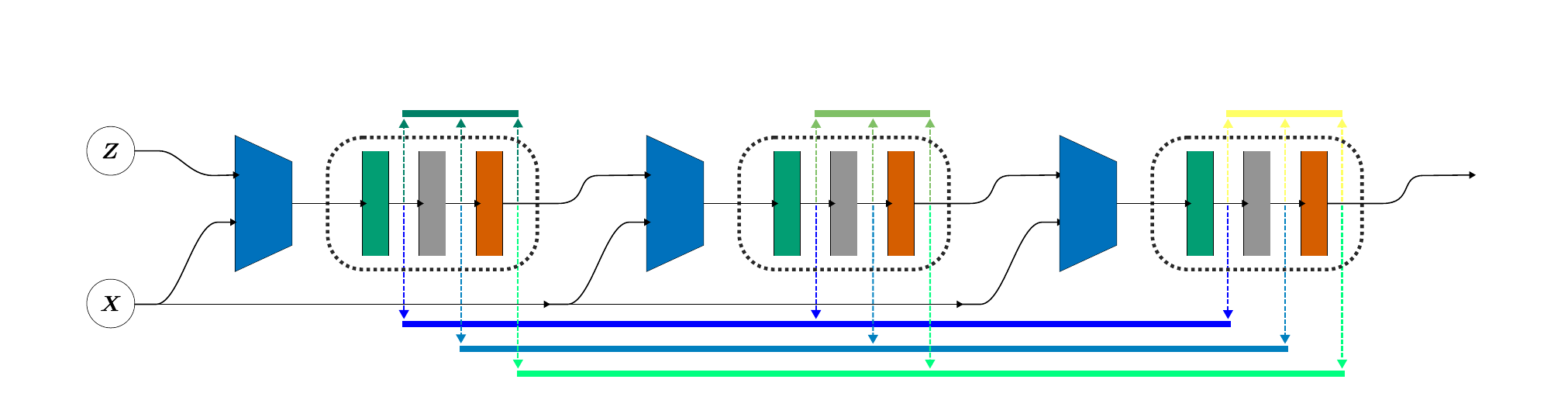}
    \caption{An illustrative \ac{LT} architecture. Rectangles represent Transformer blocks; weights are shared between rectangles of the same color. Dashed regions denote stacks of Transformer blocks: the same stack is looped, resulting here in cyclic recurrence. The two inputs $\boldsymbol{Z}, \boldsymbol{X}$ denote the initial latent state and (optional) input injection respectively. For models with input injection (\raven{}, the Retrofitted series, Parcae), $\boldsymbol{X}$ is typically the output of the prelude layers and $\boldsymbol{Z}$ is the randomly initialised latent state. For \ouro{}, without input injection or prelude layers, $\boldsymbol{X} = \boldsymbol{0}$ and $\boldsymbol{Z}$ is the input to the Transformer (token embeddings). The \textbf{\textcolor{mplsummergreen}{green}-\textcolor{mplsummeryellow}{yellow}} bars above the diagram indicate grouping metrics by recurrence and the \textbf{\textcolor{mplwinterblue}{blue}-\textcolor{mplwintergreen}{green}} bars indicate grouping by layer.}\label{fig:looped_transformer_diagram}
\end{figure}

\acp{LT} are Transformers that utilize ``recurrence in depth'' -- that is, they reapply layers to repeatedly act on the latent states. Recent research has identified that an effective method \citep{bae2025mixture} to achieve this is via \emph{cyclic recurrence}: a fixed sequence of layers is repeated in a ``cyclic'' pattern. We define a $k$-stacked cyclic block as $\stack_k(\mX) = \block_k(\block_{k-1}(\dots \block_1(\mX) \dots))$.

In the case of \citet{geiping2025scaling,mcleish2025teaching} this stacked block may also take an additional input $\mZ \in \mathbb{R}^{T \times D}$ which is typically initialized from a Normal distribution, as well as the original input to the recurrent section: this is known as \emph{input injection} \citep{bai2019deep,anil2022path}, and the two inputs are projected into common feature space $\mathbb{R}^D$ before the block is applied. In the case of input injection, a $k$-stacked block therefore becomes $\stack_k(\mX, \mZ) = \block_k ( \block_{k-1}( \dots \block_1([\mX, \mZ] \mW_I) \dots )),$ where $[\cdot, \cdot]$ denotes concatenation in the channel dimension and $\mW_I \in \mathbb{R}^{2D \times D}$ is a learned projection matrix. This allows us to define a $(k \otimes l)$-Recurrent block as a $k$-stacked block repeated $l$ times, $R_{l,k}(\mX) = {\stack_k ( \stack_{k}( \dots \stack_k(} \mX) \dots ))$, which with input-injection becomes $R_{l,k}(\mX, \mZ) = {\stack_k ( \mX, \stack_{k}( \dots \mX, \stack_k( \mX, \mZ)) \dots ))}$. An example $(3 \otimes 3)$ \ac{LT} utilizing input injection is visualised in \Cref{fig:looped_transformer_diagram}. An \ac{LT} may include feedforward layers both before and after the Recurrent block: following the convention of \citet{geiping2025scaling}, we refer to these as \emph{prelude} and \emph{coda} layers respectively. 

We will often refer to \emph{virtual layers}: each virtual layer is one execution of a transformer layer (attention and MLP) during a forward pass. Crucially, this is not the same as the number of unique layers: an \ac{LT} with 2 layers repeated 8 times has 16 virtual layers. In presenting our results, we \textbf{group} metrics computed over these virtual layers in the following ways:
\begin{itemize}[nosep]
    \item \textbf{By recurrence}: Separate lines are visualized for each complete pass through the recurrent block. The $x$-axis is percentage-scaled to represent \emph{relative} depth within that block (including prelude/coda), allowing us to overlay and compare successive passes. Shown with \textbf{\textcolor{mplsummergreen}{green}-\textcolor{mplsummeryellow}{yellow}} colorbar, with later recurrences colored more yellow.
    \item \textbf{By unique layer}: Separate lines show evolution of each layer across recurrences. Shown with \textbf{\textcolor{mplwinterblue}{blue}-\textcolor{mplwintergreen}{green}} colorbar, with later layers colored more green.
\end{itemize}

These groupings are visualised above (grouping by recurrence) and below (grouping by layer) \Cref{fig:looped_transformer_diagram}, where each color-coded line represents a plotted line with the same color.

\subsection{Stages of Inference}\label{subsec:stages_of_inference_prelim}

Layers in feedforward Transformers exhibit distinct and consistent behaviors that change sharply in depth: \citet{lad2024remarkable} originate the term ``stages of inference'' and demonstrate how different mechanisms emerge at different depths. \citet{queipo2025attention} further develop this viewpoint, proposing three distinct stages:

\textbf{Mixing:} Transformer layers vary significantly in how they ``mix'' information across tokens \citep{queipo2025attention,lad2024remarkable}. This mixing often peaks in early layers, but architectures like Llama can also exhibit high mixing in middle layers \citep[Fig. 6]{barbero2025llms}. To quantify this, we define the \emph{Mixing Score} -- the average Shannon entropy of attention pattern rows (excluding the first), normalized by the maximum entropy permitted by the causal mask:
\begin{equation}\label{eq:mixing_score}
    \text{Mixing Score} = \frac{1}{T-1} \sum_{i=2}^T \frac{H(\mA_{i, :})}{\log(i)} = \frac{1}{T-1} \sum_{i=2}^T \frac{- \sum_{j=1}^i \alpha_{ij} \log{\alpha_{ij}}}{\log(i)}.
\end{equation}

A score of 0 implies attention concentration on a single token e.g. identity, ``off-by-one'' or sink; 1 implies even distribution across all preceding tokens. In this way, Mixing Score acts as a unified heuristic for information dispersion, capturing the intuition of sink rate \citep{gu2024attention} and preceding attention weight \citep{lad2024remarkable} without specifying to a particular head taxonomy.

We highlight that this is a modified version of the Mixing Score introduced in \citet{queipo2025attention}: see \Cref{appen:additional_cyclic_attention_behavior} for the rationale behind our modifications.

\textbf{Embedding:} \citet{queipo2025attention} define this as the ``compression'' stage, where residual entropy is sharply reduced by massive activations. This compression correlates with a peak in downstream embedding task performance \citep{skean2025layer,lad2024remarkable}. While the looped models studied here do not exhibit consistent residual stream compression -- likely due to residual stream norms suppressing the required massive activations, discussed in \Cref{sec:stages_of_inference} and \Cref{appen:negative_raven_result} -- we do observe a clear alignment in downstream embedding performance. We quantify this by training linear probes at every virtual layer to map hidden representations to multiple-choice answers.

\textbf{Refinement:} Later layers are primarily responsible for refining the next-token distribution \citep{queipo2025attention,lad2024remarkable}. We track this behavior using the {logit lens} \citep{nostalgebraist2020interpreting}, projecting intermediate hidden states through the model's final normalization and unembedding layers to obtain vocabulary logits. By measuring the accuracy of these projections at every virtual layer, we observe how the looped block iteratively converges on the correct token distribution.

%% file: sections/fixed_point.tex
It is known that weight-tied Transformers can converge to a fixed point\footnote{As observed by \citet{geiping2025scaling}, other limiting behaviors can occur. In \Cref{appen:orbits_and_sliders} we show that 1) this is rare, the vast majority of tokens reach a fixed-point and 2) even with other limiting behavior, stages of inference remain constant.} $\mX' = \stack_k(\mX')$ \citep{bai2019deep,bansal2022end,anil2022path}. In this section, we explore for which looped LLMs this is true; we investigate the implications on layerwise dynamics; we prove that convergence results in consistent attention patterns; and we investigate the architectural drivers behind fixed-point behavior.

We motivate our work by noting that in cyclic architectures, this block-wise fixed point implies that every layer within the loop reaches its own fixed point:

\begin{proposition}[Cyclic recurrent blocks reach cyclic fixed points]\label{prop:cyclic_fixed_point}
    Let $(l, k)$-Recurrent block reach a fixed point $\mX'$ such that $\stack_k(\mX') = \mX'$. Then any cyclic permutation of blocks $1, \dots, k$ will also have reached a fixed point.
\end{proposition}

These fixed points are not necessarily the same: the recurrent block can trace a stable {cyclic trajectory} in latent space as illustrated in \Cref{fig:random_trajectory}. In \Cref{subsec:fixed_point_empirical} we confirm that this \emph{cyclic fixed point} behavior is observed frequently in practice. 

As the residual stream stabilizes, the attention patterns must similarly stabilize. Assuming both this converged residual stream and bounded inputs (guaranteed by pre-norm in all studied models), tied weights ensure that attention remains consistent across recurrences as shown in \Cref{prop:similar_attention_patterns}.

\begin{proposition}[Recurrent attention patterns change slowly under state convergence]\label{prop:similar_attention_patterns}
Fix a layer $\ell$ in the recurrent block and consider its attention weight matrices to be tied across recurrences (so $\mW_{Q,\ell},\mW_{K,\ell}$ are the same for all $t$). 
Let $\|\cdot\|$ denote a submultiplicative matrix norm that is invariant under transposition (e.g. Spectral or Frobenius norms). Assume the corresponding attention inputs are bounded under this norm as $\|\mX_{\ell,t}\|\le B$ for all $t$. Define $\kappa_\ell = \|\mW_{Q,\ell}\mW_{K,\ell}^\top\|$.  
Then, writing $\mathcal S_\ell(\mX):=\softmax(A_\ell(\mX))$ with $A_\ell(\cdot)$ as defined above, for any $t\ge 1$,
\begin{equation*}\label{eq:recurrent_attention_lip}
\big\|\mathcal S_\ell(\mX_{\ell,t})-\mathcal S_\ell(\mX_{\ell,t-1})\big\|
\;\le\;
L_{\mathrm{sm}}\,
\frac{2B\,\kappa_\ell}{\sqrt d}\,
\big\|\mX_{\ell,t}-\mX_{\ell,t-1}\big\|,
\end{equation*}
where $L_{\mathrm{sm}}$ is a Lipschitz constant of the row-wise softmax with respect to the chosen norm. 
\end{proposition}

\subsection{Empirical Validation}\label{subsec:fixed_point_empirical}

Throughout the main body of this paper we investigate the same four looped LLMs: \ouro{} 1.4B \citep{zhu2025scaling}, \raven{} \citep{geiping2025scaling}, retrofitted recurrent Llama \citep{mcleish2025teaching} and Parcae 770m \citep{prairie2026parcae}. These models span a range of design choices (\Cref{tab:looped_model_details}), the architectural impact of which we isolate in \Cref{subsec:norm_and_ii}. Unless stated otherwise, results are evaluated on 256 random examples from the GSM8k test set; in \Cref{appen:non_reasoning_soi} we verify that the conclusions hold on non-reasoning tasks.

\begin{figure}[ht]
    \centering
    \includegraphics[width=\linewidth]{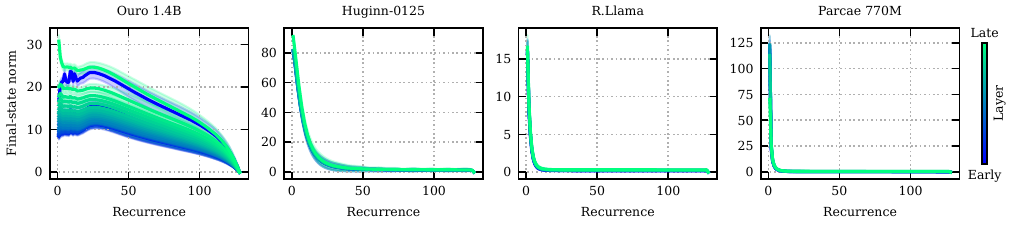}
    \caption{\textbf{Residual stream stability across recurrences.} Visualises the norm of the difference from each layer to an ``approximate fixed point'' (the 128th recurrence). While \raven{}, \rllama{} and \parcae{} converge rapidly, \ouro{} continually drifts without reaching a fixed point.}\label{fig:fixed_point_norms_main}
\end{figure}

\begin{figure}[hb]
    \centering
    \includegraphics[width=0.9\linewidth]{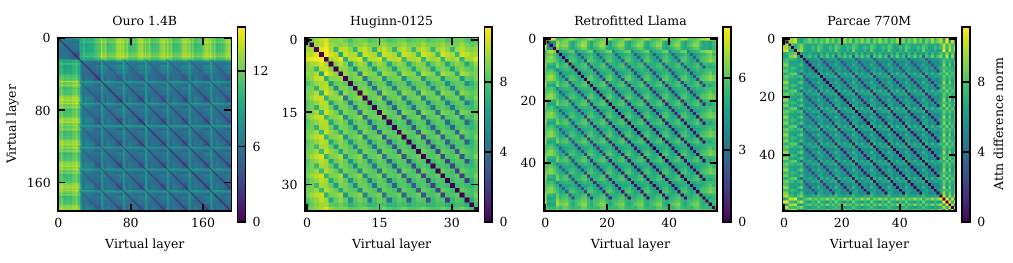}
    \vspace{-0.1cm}
    \caption{\textbf{Cyclic similarity in attention patterns.} Frobenius norms between attention patterns at all virtual layer pairs (8 loops), averaged across the batch and head dimensions. Strong diagonal patterns indicate that attention at any layer is similar across recurrences, confirming \Cref{prop:similar_attention_patterns}.}
    \vspace{-0.1cm}
    \label{fig:pairwise_attn_pattern_diff_norm_main}
\end{figure}

\begin{wrapfigure}{r}{0.5\textwidth}
\vspace{-0.5cm}
    \centering
    \includegraphics[width=0.9\linewidth]{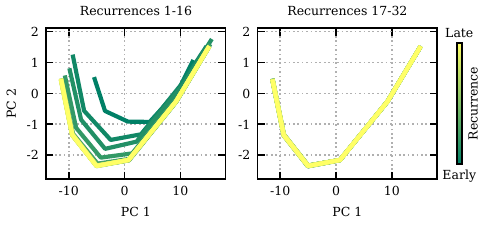}
    \vspace{-0.1cm}
    \caption{\textbf{Latent space (PCA) trajectory for Retrofitted Llama} hidden states at the final sequence position. Overlap in the second plot implies a stable cyclic fixed point.}
    \vspace{-0.3cm}
    \label{fig:retro_llama_trajectory}
\end{wrapfigure}

\Cref{fig:fixed_point_norms_main} plots for each layer $l$ the distance from the residual stream at recurrence $r$ to an approximate fixed point after 128 recurrences, $| \vh_{l, r} - \vh_{l, 128} |_2$. Of the four models tested, only \ouro{} does not reach a fixed point -- and the other models stabilize very quickly. This validates the premise of \Cref{prop:similar_attention_patterns}: as a fixed point is approached, the layerwise residual stream changes minimally between recurrences. As predicted by this result, heatmaps of the Frobenius norms between attention patterns (\Cref{fig:pairwise_attn_pattern_diff_norm_main}) over eight loops demonstrate high cyclic similarity; the distinct diagonal patterns indicate that a layer's attention pattern is most similar to its own pattern in other recurrent steps. Even though \ouro{} does not reach a fixed point, over the range visualised in \Cref{fig:pairwise_attn_pattern_diff_norm_main} it still exhibits highly similar attention patterns: \Cref{sec:stability} demonstrates that this does not hold at later recurrences.

For models that do reach a fixed point, the recurrent block effectively traces a consistent cycle in latent space: we illustrate this phenomenon via PCA for retrofitted Llama in \Cref{fig:retro_llama_trajectory}.

\subsection{Impact of Architecture Choices}\label{subsec:norm_and_ii}

To isolate stability properties arising from architectural choices from learned behaviors, we evaluate randomly initialized 12-layer models across the architectural variants found in \raven{}, \ouro{}, and retrofitted models, testing each with and without input injection. \Cref{fig:summarised_stability} plots mean results over five random seeds. We see that while input injection enables stable fixed-point behavior across most norm types, omitting it renders all configurations unstable. Note that the high cosine similarity of pre-norm without input injection is pathological: its residual stream norm grows unbounded, overwhelming the nonlinear block's contribution at each layer and creating the illusion of convergence in cosine similarity space. Curiously, we discover that the \ouro{} norm structure renders even input injection unstable.

\begin{figure}[ht]
    \centering
    \includegraphics[width=\linewidth]{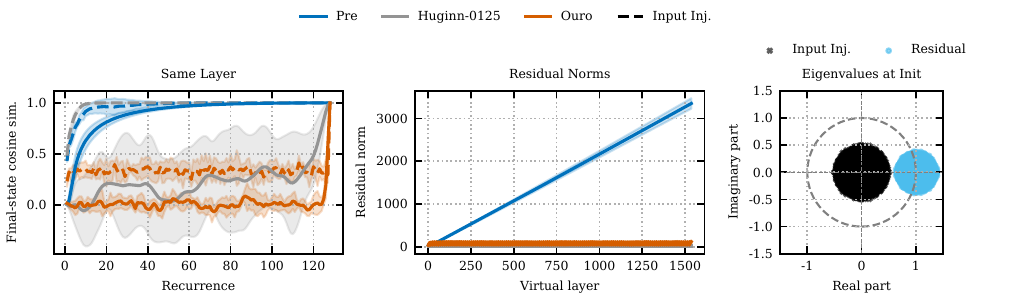}
    \caption{\textbf{Impact of Input Injection and Norm Structure on Stability.} \textbf{Left:} Cosine similarity between the first layer’s residual stream and its approximate fixed point. \textbf{Center:} Residual stream norms. \textbf{Right:} Spectral distribution of the recurrent block Jacobian for pre-norm variants, showing how input injection pulls eigenvalues within the stable unit circle. Left and center plots averaged over 5 random seeds.}
    \vspace{-0.1cm}
    \label{fig:summarised_stability}
\end{figure}

We can partially explain the stabilizing effect of input injection using the dynamical system framing of \citet{prairie2026parcae}, wherein the recurrent update is written as:
\begin{equation}\label{eq:dynamical_system}
    \vx_{t+1} = {\mA} \vx_t + {\mB} \ve + \mathcal{R}(\mX_t, \mE),
\end{equation}

where $\mathcal{R}$ captures nonlinear components, $\ve$ is the input embedding, and $\vx$ is the hidden state. Linearizing the system (omitting $\mathcal{R}$) yields a necessary convergence condition: the spectral radius $\rho(\mA)$ must be strictly less than 1. Without input injection, $\mA = \mathbb{I}$ and $\rho(\mA) = 1$, leaving the system only marginally stable. \Cref{prop:II_spectral_radius} shows that the standard initialization of input injection layers encourages $\rho(\mA) < 1$ in high dimensions.

\begin{proposition}\label{prop:II_spectral_radius}
    Assume a looped model which permits the dynamical system formulation of \Cref{eq:dynamical_system} (e.g. it has no norm on the residual stream). Assume that input injection is via projection by matrix $\mW_I \in \mathbb{R}^{D \times 2D}$ acting on the concatenated vector $[\vh_t, \ve] \in \mathbb{R}^{2D}$, and that $\mW_I$ is initialized with elements drawn from $\mathcal{N}(0, \sigma_I^2)$. If $\sigma_I < 1 / \sqrt{D}$ (as is the case for Xavier initialization \citep{glorot2010understanding}) then $\rho(\mA) < 1$ with probability $1$ as $D \to \infty$.
\end{proposition}

Empirical measurements of stable \emph{trained} models confirm that $\rho(\mA) < 1$: \raven{} (see caveat below), \rllama{}, \rtllama{} and \rolmo{} have spectral radius 0.24, 0.48, 0.45 and 0.52 respectively. This suggests that training does not necessarily push the spectral radius into the unstable regime. However, this stability is not guaranteed: we direct the reader to \citet{prairie2026parcae} for a discussion of how $\mA$ can be constrained to encourage stability.

We highlight that this linear analysis only holds for the Parcae and retrofitted models: \ouro{} and \raven{} both normalise the residual stream, and as such \Cref{eq:dynamical_system} does not hold.

%% file: sections/stages_of_inference.tex
As established in \Cref{sec:fixed_points}, looped models converge to a regime where attention patterns within individual layers stabilize across recurrences. This stability contrasts sharply with feedforward Transformers, where unconstrained weights allow for abrupt shifts in attention dynamics. Prior work links these transitions to layer behavior and specific stages of inference \citep{lad2024remarkable, queipo2025attention}; we investigate how cyclic weight-sharing reconfigures these stages within the looped architecture.

We first examine attention \textbf{mixing} behavior by plotting the \mainmixing{} across virtual depth (\Cref{fig:stages_of_inference_cyclic_and_similar}). Rather than distinct mixing behaviors distributed across the entire model depth, we find clear repeating patterns of behavior that repeat within every iteration of the recurrent block. The behavior of individual layers (visualized as solid lines) converges rapidly toward constant behavior, indicating that each unique layer specializes in a specific portion of this mixing cycle.

\begin{wrapfigure}{r}{0.496\linewidth}
    \centering    
    \includegraphics[width=\linewidth]{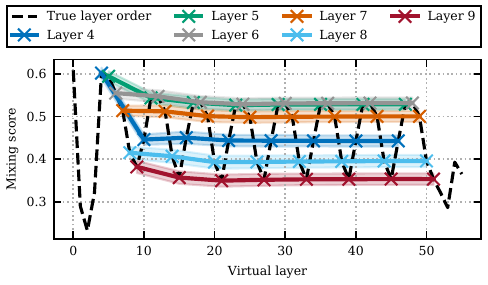}
    \vspace{-0.45cm}
    \caption{\textbf{Cyclic Mixing Stages in Retrofitted Llama.} Solid lines denote unique layers, dashed line the sequence of virtual layers across 8 recurrences. Individual layers converge to constant mixing scores, resulting in a stable, repeating cycle.}
    \vspace{-0.15cm}
    \label{fig:stages_of_inference_cyclic_and_similar}
\end{wrapfigure}

\Cref{fig:summarised_recurrent_block_stages} demonstrates that this per-cycle organization often mirrors the stages of inference found in feedforward models (\Cref{appen:mixing_score_correlations} for correlations). We observe the following: (1) \ouro{}, despite being trained from scratch with recurrence, adopts a mixing profile that closely mirrors the Llama family. (2) \raven{} does not exhibit similar mixing stages to any known Feedforward models. We attribute this to the repeated normalization of the residual stream at every layer, which suppresses the formation of massive activations that typically drive these stages \citep{queipo2025attention}. We provide extended evidence for this hypothesis in \Cref{appen:negative_raven_result}, showing that ablating the MLP responsible for the formation of massive activations results in stages of inference that deviate from those of the feedforward model. (3) Retrofitted models preserve the stages of their base architectures, with the prelude and coda handling the initial and final stages respectively, while the recurrent block repeatedly cycles the ``middle'' stages. (4) The Parcae models mirror the stages of inference of a ``feedforward control'' parameter-matched Parcae model trained with only one recurrence. These stages are different from prior analyses of older LLMs, which we attribute to architectural modifications of Nanochat, such as Value Embeddings. However, as with the \ouro{} and retrofitted results, Parcae models exhibit stages of inference that mirror those of a comparable Feedforward model.

The fact that the \ouro{} and Parcae models independently learn stages of inference during pretraining that mirror comparable feedforward models suggests that this functional decomposition is a robust emergent property. We provide further evidence for this self-organization in \Cref{appen:self_organising_soi}.

\begin{figure*}[ht]
    \centering    
    \includegraphics[width=\linewidth]{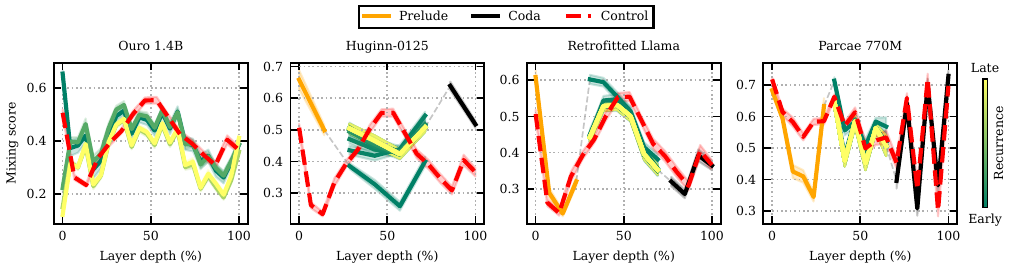}
    \vspace{-0.1cm}
    \caption{\textbf{Internal Mixing Stages per Recurrent Loop.} \ouro{} resembles Llama mixing stages of inference; the retrofitted models retain the stages of inference of their base models; Parcae resembles a feedforward control of the same architecture.}
    \vspace{-0.1cm}
    \label{fig:summarised_recurrent_block_stages}
\end{figure*}

We characterize \textbf{embedding} and \textbf{refinement} stages by evaluating per-layer accuracy on downstream tasks: \Cref{fig:linear_probe_summary,fig:logitlens_summary} respectively. These functional trajectories broadly align with feedforward observations: embedding performance peaks in middle layers, while next-token refinement accuracy increases toward the end of the block. In addition, two looped-specific behaviors emerge: \textbf{(1) Repeated ``Detokenization'':} In \ouro{}, which lacks a coda to delegate refinement to, early layers in later recurrences transform representations back \emph{out} of the refined next-token space, evidenced by the U-shaped LogitLens accuracy seen within later loops. \textbf{(2) Iterative Conceptual Refinement:} In \raven{}, embedding performance increases dramatically for the first $\sim$10 recurrences before plateauing. This demonstrates that the model uses looping to iteratively refine the latent representation in a way that is beneficial for downstream embedding tasks.

\begin{figure*}[ht]
    \centering    
    \includegraphics[width=\linewidth]{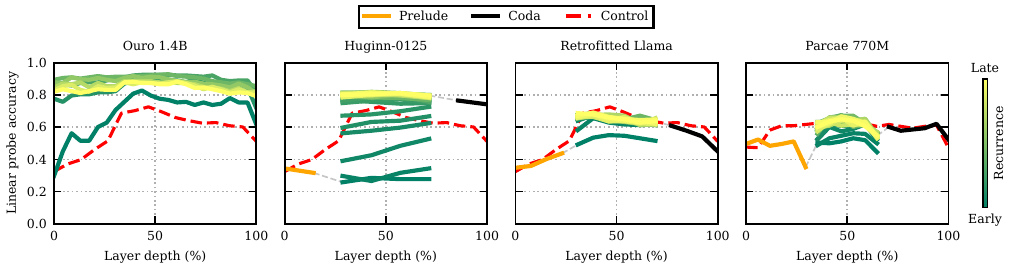}
    \vspace{-0.1cm}
    \caption{\textbf{Linear-probe multiple-choice accuracy on ARC-Easy across model depth.}
Residual stream representations for every question--answer pair are mean-pooled over all non-padding tokens; linear probes at each virtual layer are trained on the ARC-Easy training split (underlying model remains frozen); accuracy is the fraction of held-out test questions answered correctly.}
    \vspace{-0.1cm}
    \label{fig:linear_probe_summary}
\end{figure*}

\begin{figure*}[ht]
    \centering    
    \includegraphics[width=\linewidth]{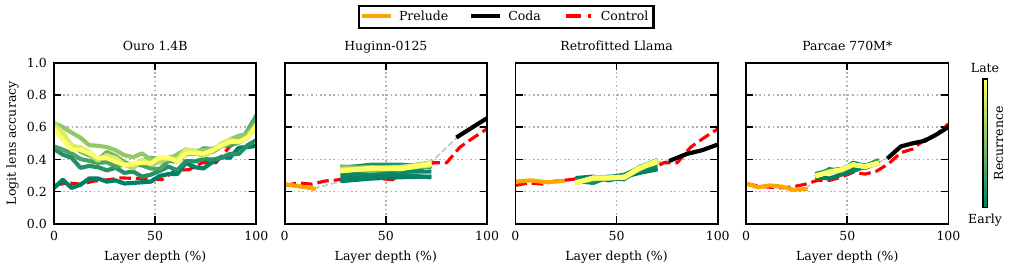}
    \vspace{-0.1cm}
    \caption{\textbf{Logit-lens multiple-choice accuracy on ARC-Easy across model depth.}
Residual stream is decoded using the model's final layer normalisation and unembedding; candidate answers are appended to the question and scored by log-probability; accuracy computed as the fraction of held-out questions answered correctly. Parcae additionally applies pre-coda linear projection.}
    \vspace{-0.1cm}
    \label{fig:logitlens_summary}
\end{figure*}

\section{Stability to Unseen Numbers of Recurrences}\label{sec:stability}

\Cref{sec:fixed_points} establishes that some \ac{LT} architectures converge to a fixed point while others do not. However, when evaluating these \acp{LT} within the range of recurrences on which they were trained (\Cref{sec:stages_of_inference}), we see both types converge to feedforward stages of inference within each recurrent block. In this section, we demonstrate that stable and unstable \acp{LT} exhibit different behavior when \emph{extrapolating} to unseen recurrent depths: models that do not reach a fixed point exhibit \textbf{unstable stages of inference}. Conversely, models with recurrent-block-wise stages of inference that also converge to a fixed point will keep enacting these stages of inference for arbitrary test time recurrences.

We first demonstrate this behavior by extending the recurrence depth of \rllama{} and \ouro{}. For retrofitted Llama (representative of models using input injection), individual layers rapidly reach stable states, ensuring the mixing cycles remain constant to arbitrary depth. By contrast, individual layers in \ouro{} continually deviate, preventing the model from maintaining a consistent functional profile. The impact on block-wise stages is visualized in \Cref{fig:stability_in_percentage_depth}: at later out of distribution recurrences, \ouro{}’s stages of inference drift significantly from their trained configuration.

\begin{wrapfigure}{r}{0.5\textwidth}
    \vspace{-0.5cm}
    \centering
    \includegraphics[width=\linewidth]{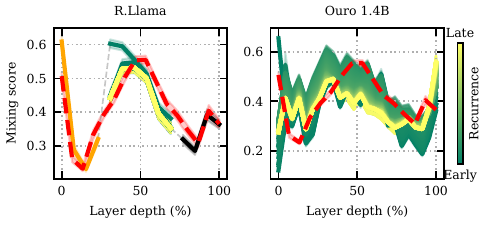}
    \vspace{-0.5cm}
    \caption{\textbf{Impact of fixed-point behavior on Stages of Inference:} stable models preserve their functional stages whereas unstable models deviate}
    \vspace{-0.3cm}
    \label{fig:stability_in_percentage_depth}
\end{wrapfigure}

 This mechanistic stability appears to correlate with a model's ability to utilize additional test-time compute. Architectures that maintain stable stages of inference avoid the performance deterioration typically associated with ``overthinking'' in deep networks \citep{bansal2022end}. For instance, existing work shows that while extending recurrence beyond the training budget harms the performance of the unstable \ouro{} \citep[Tab. 10]{zhu2025scaling}, the performance of the stable \raven{} remains constant or improves in the same extrapolation regime \citep[Fig. 1]{geiping2025scaling}. These existing results suggest that functional consistency -- underpinned by fixed-point convergence -- encourages reliable test-time scaling.

 We further test this hypothesis by attempting to \emph{isolate} the stability of a Looped LLM with minimal other changes to the training recipe. To this end we define a variant of the Parcae 140M \citep{prairie2026parcae} model in which we replace the $\mA$ matrix (the state-transition that carries the residual stream between recurrences, see \Cref{eq:dynamical_system}) with a single $\rho$ scalar, allowing us direct (and trivial) control over the linearized spectral radius. We train these models with Poisson-sampled recurrent depth, mean 8 and capped at 16. As we sweep this $\rho$ value from 0.25 up to 1.0, visualized in \Cref{fig:stability_summary_scaling}, we see that a spectral radius of 1.0 results in an unstable system, with continued residual growth and perplexity that increases outside of the depth range seen at training. \Cref{fig:stability_summary_mixing} links this failure of test-time depth extrapolation to a loss of stable stages of inference: the mixing stages of the $\rho=0.25$ model remain stable, adhering closely to the stages seen at the largest training depth, while those of the $\rho=1.0$ model drift significantly.

 \begin{figure}[ht]
    \centering
    \begin{subfigure}{0.48\textwidth}
        \centering
        \includegraphics[width=\linewidth]{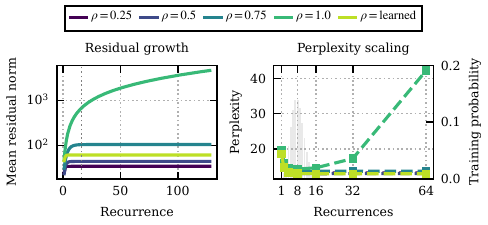}
        \caption{Residual and performance scaling}
        \label{fig:stability_summary_scaling}
    \end{subfigure}\hfill
    \begin{subfigure}{0.48\textwidth}
        \centering
        \includegraphics[width=\linewidth]{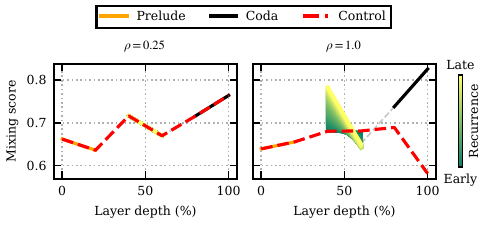}
        \caption{Stages of Inference stability}
        \label{fig:stability_summary_mixing}
    \end{subfigure}
    \caption{\textbf{Isolated impact of stability.} A controlled set of Parcae-style models where we directly control the spectral radius $\rho$ of the residual state transition. Stable models exhibit stable residual norms and consistent perplexity in the extrapolation regime (\Cref{fig:stability_summary_scaling}), and consistent stages of inference in the extrapolation regime (\Cref{fig:stability_summary_mixing}, left). Unstable models exhibit the opposite. Models are trained with Poisson-sampled recurrences of mean 8 and capped at 16, visualized in grey in \Cref{fig:stability_summary_scaling}, right. The controls in \Cref{fig:stability_summary_mixing} are the mixing scores at the 16th recurrence: the largest seen at train time. Overlaid on this are the mixing stages for all recurrences \emph{after} the 16th, up to 128 recurrences. ``Learned'' $\rho$ follows the standard Parcae 140m model architecture with a per-channel diagonal $\mA$ matrix.}
    \label{fig:isolating_stability_summarised}
\end{figure}

%% file: sections/related_work.tex
\paragraph{Looped and Recurrent Transformers.}

Weight-tied Transformers originated with Universal Transformers \citep{dehghani2018universal}, with recent extensions focusing on sparse/conditional computation \citep{tan2023sparse,csordas2024moeut} and reasoning-centric designs like HRM \citep{wang2025hierarchical} and TRM \citep{jolicoeur2025less}. Within language modeling, we highlight \raven{} \citep{geiping2025scaling}, \ouro{} \citep{zhu2025scaling}, and Mixture-of-Recursions \citep{bae2025mixture} as models that have been pretrained from random initialization, as well as recent work by \citet{mcleish2025teaching, koishekenov2025encode} that retrofit recurrence into pretrained LLMs.

In terms of mechanistic studies, we highlight \citet{pappone2025two} who analyze latent dynamics in recurrent Transformers. However, the setting of their analysis differs from ours: they analyze a model in which each recurrent block comprises either 1 or 2 layers, and the model comprises multiple \emph{separate} recurrent blocks. We instead study a single looped block with deeper cycles (4+ layers) and analyze the \emph{internals} of these cyclic blocks by examining the latent states of each separate layer. Other relevant work addresses looped model expressivity \citep{xu2024expressive, saunshi2025reasoning} and fixed-point stability \citep{bai2019deep, anil2022path, ke2024advancing, yudin2025pay}. The only work of which we are aware that analyses the internal states of the cyclic recurrent blocks is \citet{lu2025latent}, who demonstrate cyclic behavior in logit lens prediction throughout recurrent blocks, but do not analyse this ``per-recurrence'' or draw the comparison to feedforward models. We highlight also the work of \citet{fumero2026navigating}, who examine the latent dynamics of autoencoders, with additional theoretical analysis around fixed point behavior. 

\paragraph{Stages of Inference and Attention Dynamics.}

The idea that LLMs organize their feedforward computation into several distinct \textit{stages of inference} was first proposed by \citet{lad2024remarkable}. Building on this, \citet{queipo2025attention} explain the emergence of these stages through the behavior of attention heads, driven by massive activations \citep{sun2024massive}. We also highlight a complementary line of work that analyzes LLM learning dynamics via attention patterns through the lens of \textit{mixing} \citep{barbero2024transformers, arroyo2026a, velivckovic2024softmax, barbero2025llms}, motivated by information propagation challenges originally studied in Graph Neural Networks (GNNs) \citep{cai2020note,alon2020bottleneck,arroyo2025vanishing,hariri2025return,blayney2025glstm}.

\paragraph{Test-time Computation.}

Recurrence is a primary mechanism for scaling compute at test time, offering an alternative to early-exit \citep{schuster2022confident} or continuous-thought architectures \citep{darlow2025continuous}. Classic approaches to adaptive test-time compute include Adaptive Computation Time \citep{graves2016adaptive} and subsequent probabilistic halting frameworks such as PonderNet \citep{banino2021pondernet}. Building on these foundations, recent work has begun to characterize when additional inference compute actually generalizes beyond training-time budgets \citep{schwarzschild2021can}, how to mitigate failures due to excessive computation (``overthinking'') \citep{bansal2022end}, how to ensure stable dynamics with repeated iteration \citep{bear2024rethinking}, and how \acp{LT} can better generalize to out of distribution tasks at test time \citep{mcleish2024transformers}.

%% file: sections/conclusion.tex
\paragraph{Summary.}

This paper examines the limiting behavior of \acp{LT}, exploring when and why fixed points are reached. We demonstrate across a range of architectures that recurrent blocks tend to ``mirror'' the stages of a feedforward Transformer. We investigate the implications for these mixing stages when models converge to a stable fixed point, and when they do not.

\paragraph{Implications of Findings.}

The implications of our findings extend to theoretical understanding of Transformers and practical design of looped architectures. By decoupling functional depth from parameter count, looped models provide a control for investigating \emph{why} inference stages emerge. While prior work suggested these stages serve to mitigate the cumulative ``harms'' of extreme depth \citep{barbero2025llms, queipo2025attention}, our results demonstrate that looped models can develop these same stages while improving performance through deeper recurrence. Our work shows that architectural insights from feedforward models may be transferable to the looped setting, such as stage-dependent attention sparsification or leaner, low-rank parameterization of middle-stage MLPs where representations are most compressed. We also provide evidence of the link between stability and consistent stages of inference, showing that without stability functional behavior can drift across recurrences. Finally, our analysis identifies specific architectural constraints: prelude and coda layers appear useful to delegate embedding and unembedding behavior, but normalization applied to the residual stream can be counterproductive, suppressing massive activations and mixing behavior.

\paragraph{Limitations and Future Work.}

We focus exclusively on cyclic recurrence as this appears to be the dominant approach in the literature. However, this means our analysis does not extend to sequential recurrence with multiple separate recurrent blocks; for analysis of this setting we refer the reader to \citet{pappone2025two}. Several limitations remain open for future work: we have limited theoretical understanding of when and why \acp{LT} are stable; and it remains unclear whether this stable limiting behavior is desirable or restrictive for reasoning tasks.

%% file: sections/appendix.tex
\newpage

\section{Potential Societal Impacts}
The potential for negative societal impacts of this work is minor: we seek to further understand existing looped language models, but no new models or techniques are introduced in this work.

However, we believe it is relevant to highlight that this research field of exploring recurrent depth for improved reasoning and test time compute has a clear negative outcome that must be avoided: by moving inference time compute away from the current standard of Chain of Thought into latent space we lose controllability and interpretability over the ``thoughts'' of these LLM systems. More research must be done on retaining this control before systems are deployed at scale. 

\section{Proofs of Propositions}
\input{sections/appendices/proofs.tex}

\section{Additional Experimental Details}\label{appen:experiment_details}

\input{sections/appendices/additional_experiment_details.tex}

\FloatBarrier
\section{Non-Fixed-Point Limiting Behavior}\label{appen:orbits_and_sliders}
\input{sections/appendices/orbits_sliders.tex}

\FloatBarrier
\section{Additional Fixed Point Results}\label{appen:additional_fixed_point_results}
\input{sections/appendices/additional_fp_results.tex}

\FloatBarrier
\section{Additional Latent Trajectories}\label{appen:additional_latent_trajectories}
\input{sections/appendices/additional_trajectory_results.tex}

\FloatBarrier
\section{Additional Stages of Inference Results}\label{appen:additional_soi_results}

\input{sections/appendices/additional_soi_results.tex}

\FloatBarrier
\section{Code, Data and Model Licenses}\label{appen:licenses}

\input{sections/appendices/licenses.tex}

%% file: sections/appendices/proofs.tex
\paragraph{Proof of \Cref{prop:cyclic_fixed_point}}

\begin{sketch}
    Note
    $$\stack_k(\mX') = \block_k(\block_{k-1}( \dots \block_1(\mX') \dots )) = \mX'$$
    Applying $\block_1$ to both sides yields
    $\block_1(\block_k(\block_{k-1}( \dots \block_1(\mX') \dots ))) = \block_1(\mX').$
    Defining $\mY' = \block_1(\mX')$, we obtain
    $\block_1(\block_k(\block_{k-1}( \dots \block_2(\mY') \dots ))) = \mY'$. 
    The general proof follows by induction, and extends trivially to input injection.   
\end{sketch}

\begin{proof}
    Assume that $(l,k)$-Recurrent block $S_k$ reaches a fixed point such that $S_k(\mX') = \mX'$. Note that
    $$\stack_k(\mX') = \block_{k-1}(\block_{k-2}( \dots \block_0(\mX') \dots )) = \mX'$$
    define the cyclic shift function $f_n(i) = (i + n) \mod k$. Now we aim to prove by induction
    $$\forall n \in \mathbb{Z}_+ \cup \{0\} : \block_{f_n(k-1)}(\block_{f_n(k-2)}( \dots \block_{f_n(0)}(\mZ') \dots )) = \mZ'$$
    for some $\mZ'$. The base case $n=0$ is trivial as $\forall i < k: f_0(i) = i$.
    Now assume for some $n=j$
    \begin{equation}\label{eq:ind_hyp}
        \block_{f_j(k-1)}(\block_{f_j(k-2)}( \dots \block_{f_j(0)}(\mZ') \dots )) = \mZ'
    \end{equation}
    Now, let $n=j+1$. Note 
    $$f_{j+1}(i) = (i + j + 1) \mod k = f_{j}(i + 1)$$
    Therefore
    \begin{align}
        \block_{f_{j+1}(k-1)}(\block_{f_{j+1}(k-2)}( \dots \block_{f_{j+1}(0)}(\mY) \dots )) &= \block_{f_j(k)}(\block_{f_j(k-1)}( \dots \block_{f_j(1)}(\mY) \dots )) \\
        &= \block_{f_j(0)}(\block_{f_j(k-1)}( \dots \block_{f_j(1)}(\mY) \dots ))  \label{eq:cycle_applied}
    \end{align}
     Now take \Cref{eq:ind_hyp} and apply $\block_{f_j(0)}$ to both sides, defining a new fixed point $\mZ^{\prime \prime} = \block_{f_j(0)}(\mZ')$:
     \begin{align}
         \block_{f_j(0)}(\block_{f_j(k-1)}(\block_{f_j(k-2)}( \dots \block_{f_j(0)}(\mZ') \dots ))) &= \block_{f_j(0)}(\mZ') \\
         \block_{f_j(0)}(\block_{f_j(k-1)}(\block_{f_j(k-2)}( \dots \block_{f_j(1)}(\mZ^{\prime \prime}) \dots ))) &= \mZ^{\prime \prime} \label{eq:new_fixed_point}
     \end{align}
     Combining \Cref{eq:new_fixed_point} and \Cref{eq:cycle_applied}, we see that there exists a fixed point $\mZ^{\prime \prime}$ such that
     $$\block_{f_{j+1}(k-1)}(\block_{f_{j+1}(k-2)}( \dots \block_{f_{j+1}(0)}(\mZ^{\prime \prime}) \dots )) = \mZ^{\prime \prime}$$
     Therefore completing the induction step and proving the proposition.
\end{proof}

\paragraph{Proof of \Cref{prop:similar_attention_patterns}}
\begin{proof}
Define
$$\mathcal S_\ell(\mX):=\softmax(A_\ell(\mX)) = \softmax{\left ( \frac{\mX \mW_Q \mW_K^{\top} \mX^{\top}}{\sqrt{d}} \right )}$$
Let $L_{\mathrm{sm}}$ be the Lipschitz constant of the row-wise softmax such that
    \begin{equation}\label{eq:softmax_lip}
        \big\|\mathcal S_\ell(\mX_{\ell,t})-\mathcal S_\ell(\mX_{\ell,t-1})\big\| \leq L_{\mathrm{sm}} \big\| A_\ell(\mX_{\ell,t})- A_\ell(\mX_{\ell,t-1})\big\|
    \end{equation}
    Recently \citet{nair2025softmax} show that $L_{\mathrm{sm}} = 1/2$.

    Define for convenience $\mM = \mW_Q \mW_K^{\top}$ and $\Delta_{\ell, t} = \mX_{\ell,t}- \mX_{\ell,t-1}$. Then
    \begin{align*}
        A_\ell(\mX_{\ell,t})- A_\ell(\mX_{\ell,t-1}) &= \frac{\mX_{\ell,t} \mW_Q \mW_K^{\top} \mX^{\top}_{\ell,t}}{\sqrt{d}} - \frac{\mX_{\ell,t-1} \mW_Q \mW_K^{\top} \mX^{\top}_{\ell,t-1}}{\sqrt{d}} \\
        &= \frac{1}{\sqrt{d}} \left ( \left ( \Delta_{\ell, t} + \mX_{\ell, t-1} \right ) \mM \left ( \Delta_{\ell, t} + \mX_{\ell, t-1} \right )^{\top} - \mX_{\ell,t-1} \mM \mX^{\top}_{\ell,t-1} \right ) \\
        &= \frac{1}{\sqrt{d}} \Bigg (  {\Delta_{\ell, t} \mM \Delta_{\ell, t}^{\top}} + \Delta_{\ell, t} \mM \mX_{\ell, t-1}^{\top} +  \mX_{\ell, t-1} \mM \Delta_{\ell, t}^{\top} \\ &+ \cancel{\mX_{\ell, t-1} \mM \mX_{\ell, t-1}^{\top}} - \cancel{\mX_{\ell,t-1} \mM \mX^{\top}_{\ell,t-1}} \Bigg )  \\
        &= \frac{1}{\sqrt{d}} \left (  \Delta_{\ell, t} \mM \left ( \Delta_{\ell, t} + \mX_{\ell, t-1} \right )^{\top} +  \mX_{\ell, t-1} \mM \Delta_{\ell, t}^{\top} \right ) \\
        &= \frac{1}{\sqrt{d}} \left (  \Delta_{\ell, t} \mM \mX_{\ell, t}^{\top} +  \mX_{\ell, t-1} \mM \Delta_{\ell, t}^{\top} \right ) \\
    \end{align*}
    Therefore,
    \begin{align*}
        \big\|A_\ell(\mX_{\ell,t})- A_\ell(\mX_{\ell,t-1})\big\| &= \big\| \frac{1}{\sqrt{d}} \left (  \Delta_{\ell, t} \mM \mX_{\ell, t}^{\top} +  \mX_{\ell, t-1} \mM \Delta_{\ell, t}^{\top} \right ) \big\| \\
        &\leq \frac{1}{\sqrt{d}} \left ( \big\| \Delta_{\ell, t} \mM \mX_{\ell, t}^{\top} \big\| + \big\| \mX_{\ell, t-1} \mM \Delta_{\ell, t}^{\top} \big\| \right ) \\
        &\leq \frac{\big\| \mM \big\| \left ( \big\| \mX_{\ell, t} \big\| + \big\| \mX_{\ell, t-1} \big\| \right )}{\sqrt{d}} \big\| \Delta_{\ell, t} \big\|
    \end{align*}
    Assuming $\|\mX_{\ell,t}\|\le B$ for all $t$, and defining $\kappa_\ell = \|\mW_{Q,\ell}\mW_{K,\ell}^\top\| = \big\| \mM \big\|$, we see
    \begin{equation}\label{eq:a_inequality}
        \big\|A_\ell(\mX_{\ell,t})- A_\ell(\mX_{\ell,t-1})\big\| \leq \frac{2 B \kappa_\ell}{\sqrt{d}} \big\|\mX_{\ell,t}-\mX_{\ell,t-1}\big\|
    \end{equation}
    Combining \Cref{eq:softmax_lip} and \Cref{eq:a_inequality}, we complete the proof:
    \begin{equation}
        \big\|\mathcal S_\ell(\mX_{\ell,t})-\mathcal S_\ell(\mX_{\ell,t-1})\big\| \leq L_{\mathrm{sm}} \frac{2 B \kappa_\ell}{\sqrt{d}} \big\|\mX_{\ell,t}-\mX_{\ell,t-1}\big\|
    \end{equation}
\end{proof}

\paragraph{Proof of \Cref{prop:II_spectral_radius}}

\begin{proof}
    Note $\mA = \mW_I[:, :D]$ is a random matrix with elements drawn from $\mathcal{N}(0, \sigma_I^2)$. By the circular law \citep{bai1984limiting,alt2021spectral} as $D \to \infty$, $\rho(\mA) \leq \sigma_I \sqrt{D}$ with probability $1$. Therefore, if $\sigma_I < 1 / \sqrt{D}$ then $\rho(\mA) < 1$.
\end{proof}

%% file: sections/appendices/additional_experiment_details.tex
Important architectural details for each of the models used in the paper are summarised in \Cref{tab:looped_model_details}. All pretrained models are obtained from Huggingface, model references provided in \Cref{tab:model_ids_and_bases}. We use standard settings for the tokenizers of each model, and as such some models prepend a BOS token whereas others do not: we make this clear in the `Prepends BOS' column of the same table. 

\begin{table*}[ht]
    \centering
    \caption{\ac{LT} architecture summary. Retrofitted Llama and OLMo use 6 layers in their recurrent block, TinyLlama uses 8. We combine and adapt the nomenclature of \citet{geiping2025scaling,saunshi2025reasoning} and refer to an \ac{LT} with $p$ prelude layers, $k$ recurrent layers and $c$ coda layers with the tuple $(p, k, c)$. Where input injection is used, we add an $I$ subscript $(p, k, c)_I$, and when referring to a recurrent block looped a specific number of times $l$ we write $(p, k \otimes l, c)$.}
    \label{tab:looped_model_details}
    \begin{tabularx}{\textwidth}{XXccc}
        \toprule
        \textbf{Model Name} & \textbf{Structure} & \textbf{Train Loops} & \textbf{Block Norm} & \textbf{Norm After Loop?} \\
        \midrule
        Ouro 1.4B \citep{zhu2025scaling} & $(0,24,0)$ & 4 & 
        $\begin{aligned}
            \hat{\mX} &=  \mX + n \left( \attn(n(\mX)) \right) \\
            {\mX}' &=  \hat{\mX} + n \left( \text{MLP}(n(\hat{\mX})) \right)
        \end{aligned}$ & Yes \\ 
        \addlinespace %
        \raven{} \citep{geiping2025scaling} & $(2,4,2)_I$ & 32 & 
        $\begin{aligned}
            \hat{\mX} &= n \left( \mX + \attn(n(\mX)) \right) \\
            {\mX}' &= n \left( \hat{\mX} + \text{MLP}(n(\hat{\mX})) \right)
        \end{aligned}$ & Yes \\ 
        \addlinespace %
        Retrofitted \citep{mcleish2025teaching} & $(4,6,4)_I$ / $(4,8,4)_I$ & 32 & 
        $\begin{aligned}
            \hat{\mX} &= \mX + \attn(n(\mX)) \\
            {\mX}' &= \hat{\mX} + \text{MLP}(n(\hat{\mX})) 
        \end{aligned}$ & No \\
        \addlinespace %
        Parcae \citep{prairie2026parcae} & $(2,2,2)_I$ / $(4,4,4)_I$ / $(6,6,6)_I$ & 8 & 
        $\begin{aligned}
            \hat{\mX} &= \mX + \attn(n(\mX)) \\
            {\mX}' &= \hat{\mX} + \text{MLP}(n(\hat{\mX})) 
        \end{aligned}$ & No \\
        \bottomrule
    \end{tabularx}
\end{table*}

\begin{table}[!ht]
\centering
\caption{Additional Huggingface details on pretrained Looped models used.}
\label{tab:model_ids_and_bases}
\footnotesize
\begin{tabularx}{\textwidth}{ccXXcX}
\toprule
\textbf{Model} & \textbf{Num Params} & \textbf{Huggingface ID} & \textbf{Base Model} & \textbf{Prepends BOS} & \textbf{Notes} \\ 
\midrule
\ouro{} 1.4B & 1.4B & \path{ByteDance/Ouro-1.4B} & - & $\cross$ & - \\ 
\ouro{} 2.6B & 2.6B & \path{ByteDance/Ouro-2.6B} & \path{ByteDance/Ouro-1.4B} & $\cross$ & ``Upcycled'' from \ouro{} 1.4B by repeating the same layers after the first training phase. \\ 
\raven{} & 3.5B & \path{tomg-group-umd/huginn-0125} & - & $\checkmark$ & - \\ 
Retrofitted Llama & 1B & \path{smcleish/Recurrent-Llama-3.2-train-recurrence-32} & \path{meta-llama/Llama-3.2-1B} & $\checkmark$ & Models trained with fewer recurrences exhibit similar mixing patterns. \\ 
Retrofitted OLMo-2 & 1B & \path{smcleish/Recurrent-OLMo-2-0425-train-recurrence-32} & \path{allenai/OLMo-2-0425-1B} & $\cross$ & As above. \\  
Retrofitted TinyLlama & 0.8B & \path{smcleish/Recurrent-TinyLlama-3T-train-recurrence-32} & \path{TinyLlama/TinyLlama-1.1B-intermediate-step-1431k-3T} & $\checkmark$ & As above. \\ 
Parcae & 0.14/0.37/0.77B & \path{SandyResearch/parcae-(size)} & - & $\cross$ & - \\ 
\bottomrule
\end{tabularx}
\end{table}

\paragraph{Mixing Experiments}

Unless stated otherwise, all of our mixing experiments are averaged over the same subset of 256 random examples from the test split of the GSM8k \citep{cobbe2021gsm8k} dataset. Each example is presented zero-shot, without few-shot demonstrations, system prompts or chat templates, as the question followed by its unmodified reference solution:
\begin{center}
\verb|Question: {question}\nAnswer:{solution}|
\end{center}
where \verb|\n| is a newline. The solution keeps GSM8k's \verb|<<...>>| calculator annotations and its final \verb|#### {answer}| line.

The illustrative trajectory plot of \Cref{fig:retro_llama_trajectory} is instead produced with a \emph{test sequence} that we obtain from \citet{barbero2025llms}: ``Hello! I've been well. I hope that you're doing well.'' Additional trajectory plots using random samples from GSM8k can be found in \Cref{appen:additional_latent_trajectories}.

Additional results targeting non-reasoning behavior using the HellaSwag dataset (following an identical setup of running inference on the same 256 random examples from the test split) can be found in \Cref{appen:non_reasoning_soi}.

\paragraph{Downstream (Embedding and Refinement) Experiments}

All of the ``downstream'' tasks presented in this paper are evaluated on multiple-choice tasks. In the main body of the paper, in \Cref{fig:logitlens_summary,fig:linear_probe_summary} we present results on ARC-Easy and in \Cref{appen:cyclic_functional} we present results on additional datasets: ARC-Challenge, HellaSwag and WinoGrande. We note that for the ARC datasets we remove the few questions which do not have four answers: this removes fewer than 1\% of questions in all of the splits used. We evaluate on the test set of the ARC datasets, and the validation sets for HellaSwag and WinoGrande.

Our experimental setting closely follows that of \citet{queipo2025attention}. For all multiple-choice benchmarks, we represent each answer option as an independent prompt: for ARC and HellaSwag, we concatenate the question context with the option text, while for WinoGrande each option is already a complete filled-in sentence. These option prompts are passed through the frozen model to extract hidden states at every virtual layer.

To measure \emph{embedding} representational quality at each layer and recurrence (the \textbf{linear probe} experiments), per virtual layer we train a probe: an unscaled (non elementwise-affine, per \href{https://docs.pytorch.org/docs/2.14/generated/torch.nn.LayerNorm.html}{PyTorch docs}) LayerNorm followed by a Linear layer with bias. The output dimension of this linear layer is 1: it maps to a single scalar. The input to this probe is a mean pool taken over every token index in the sequence, thus aggregating the latent representations. This results in one scalar (logit) per answer: we train these logits with cross entropy loss against the one-hot encoded true multiple choice answer. We use Adam optimizer with learning rate $5\times10^{-4}$, and train on the entire train sets of the ARC datasets, and a random 5000 sample split of the HellaSwag and WinoGrande train sets. Accuracy is computed against the test/validation set via selecting the option with the highest probe score.

To measure the next token distribution \emph{refinement} (the \textbf{logit lens} experiments) we project hidden states into vocabulary log-probabilities using the model's final layernorm and pre-trained unembedding weights. We score each answer option by summing the log-probabilities assigned to its constituent tokens, conditioned on the preceding context. LogitLens accuracy is defined as the fraction of examples where the highest-scoring option corresponds to the ground-truth answer. Parcae models have an additional learned linear projection between the core block and the coda layers: for all layers before the coda we compute the logit-lens by applying first this linear map before the final norm and unembedding. We denote this difference with an asterisk.

\paragraph{Isolating Stability Experiments}

For the experiments in \Cref{fig:isolating_stability_summarised} we train a variety of models based on the 140M Parcae \citep{prairie2026parcae} configuration. Typically the Parcae models learn a diagonal $\mA$ matrix such that the spectral radius of this matrix is constrained to $<1$ -- we instead replace this matrix with a single scalar $\rho$, thus directly (trivially) controlling the spectral radius. We train each of these models with a Poisson recurrence schedule with mean 8, and we cap the number of recurrences at 16, allowing us to control when the models are operating ``out of distribution''. We test $\rho = 0.25, 0.5, 0.75, 1.0$ and a learned variant, for which we retrain a model using the standard Parcae 140M architecture but with a maximum recurrence of 16.

\paragraph{Self-Organizing Nanochat Experiments}

Our small training runs in \Cref{appen:self_organising_soi} are performed by adapting a publicly available fork of Nanochat \citep{nanochat}, \url{https://github.com/TrelisResearch/nanochat/tree/recursive}. For all experiments we use a model dimension (residual stream) of 512 (4 heads of dimension 128) and train for 3.7B tokens. As discussed in \Cref{appen:self_organising_soi}, loss is the same as that of a regular feedforward model: cross entropy loss on the final output representation (as opposed to the summed loss of \citet{zhu2025scaling}). Each model is trained for a \emph{constant} 4 recurrences (as opposed to the Poisson sampling of \citet{geiping2025scaling}). All models use pre-norm only: this norm does not result in the stability issues reported by \citet{geiping2025scaling,zhu2025scaling}, but this is likely because we are operating at a far smaller scale.

\paragraph{Compute Requirements}

All stages of inference experiments are easily reproducible and can run in less than 12 hours on an A10 GPU. Each of the Nanochat training runs took fewer than 12 hours on an H100 GPU.

%% file: sections/appendices/orbits_sliders.tex
\subsection{How Frequent is Non-Fixed-Point Behavior?}

In this section we investigate more closely the ``orbits'' and ``sliders'' initially observed by \citet{geiping2025scaling}. These are important as they appear to represent stable limiting behavior that is \textbf{not a fixed point}.

We develop a heuristic algorithm to detect these behaviors, presented in \Cref{alg:orbit_slider}. We use the sequence of cosine similarities from the final recurrent layer (per token), where the output residual stream after each of 128 recurrences is compared to the final residual stream. This is visualized in the leftmost column of \Cref{fig:orbit_algorithm_examples}. We set threshold $\tau=0.05$ and fixed-point fraction $\rho=0.9$.

Using this algorithm, we classify the limiting behavior over \emph{all} tokens in the GSM8k test set for the \raven{} and \rllama{} models. We discover that the \emph{system prompt} used before presenting the GSM8k question has a large impact on the limiting behavior types and as such we test the following prompts across both models:
\paragraph{Long Persona} This is the system prompt used by \citet{geiping2025scaling} to produce their Orbit plots:
\begin{tcolorbox}[colback=gray!10, colframe=gray!50, arc=5pt]
You are Huginn, an AI assistant who embodies careful thought and deliberation. Your responses demonstrate:

Methodical reasoning, breaking complex problems into clear steps

Mathematical and programming expertise grounded in fundamentals

The ability to acknowledge uncertainty and correct course when needed

Clear communication that illuminates rather than just informs

When engaging with questions, you first seek to understand their deeper structure before answering. Like your namesake who flew the nine worlds seeking wisdom, you explore problems from multiple angles, helping users build genuine understanding rather than providing shallow answers.
You express warmth and intellectual curiosity while maintaining professionalism. When faced with errors or confusion, you model honest reflection and careful correction. Your goal is not just to provide answers, but to help humans develop clearer, deeper thinking.
\end{tcolorbox}

\paragraph{Long Persona (Padded)} This is the same prompt as above but with all tokens replaced with the padding token: we include this to test whether the behavior arises purely due to the length of the input.

\paragraph{Short Math} This is the system prompt used by \citet{geiping2025scaling} in their GSM8k benchmark evaluation:
\begin{tcolorbox}[colback=gray!10, colframe=gray!50, arc=5pt]
You are a helpful assistant that is capable of helping users with mathematical reasoning.
\end{tcolorbox}

We present the percentage of GSM8k tokens that exhibit each classification of limiting behavior in \Cref{tab:behavior_tokens} and the percentage of GSM8k examples that exhibit each classification of limiting behavior \emph{across any of their tokens} in \Cref{tab:behavior_existence}.

These results reveal that these non-fixed-point limiting behaviors appear to be extremely rare in practice: without a system prompt (the setting used throughout this paper) only approximately 0.02\% of tokens exhibit non-fixed-point behavior. This percentage can be significantly increased with the longer system prompt, but these behaviors remain rare at 0.14\%. Curiously, the ``Persona'' system prompt used by \citet{geiping2025scaling} seems to increase the occurrence of these orbits and sliders more than a comparable prompt of padding tokens, suggesting that the effect is not purely to do with sequence length.

We highlight that these exact results need to be treated with caution: the absolute values vary significantly depending on the algorithm hyperparameters chosen, and in the absence of a good external metric we selected these hyperparameters manually via visual inspection of the trajectories and cosine similarities to ensure the results agreed with intuition. However, across different hyperparameters the results consistently showed that non-fixed-point behavior is rare and that the long persona prompt results in a greater rate of non-fixed-point behavior. We leave in-depth classification and explanation of these phenomena to future work.

\begin{algorithm}
\caption{Per-token limiting behaviour classification}\label{alg:orbit_slider}
\begin{algorithmic}[1]
\REQUIRE similarity (or norm) series $\mathbf{s} \in \mathbb{R}^{n}$ for one token
         over the last $n$ recurrent firings (final firing excluded);
         threshold $\tau$; fixed-point fraction $\rho$
\ENSURE  Label $\ell \in \{\textsc{FixedPoint},\, \textsc{Orbit},\, \textsc{Slider},\, \textsc{Unknown}\}$

\STATE \textit{--- Detrend ---}
\STATE Fit linear trend: $[\hat{a},\hat{b}] \gets \arg\min_{a,b}\sum_i (s_i - a i - b)^2$
\STATE $\tilde{\mathbf{s}} \gets \mathbf{s} - (\hat{a}\,\mathbf{t} + \hat{b})$
       \hfill $\triangleright$ $\mathbf{t} = [0, 1, \dots, n-1]^\top$

\STATE \textit{--- Spectral amplitude ---}
\STATE $\mathbf{w} \gets \tilde{\mathbf{s}} \odot \operatorname{Hann}(n)$
       \hfill $\triangleright$ reduce spectral leakage
\STATE $\mathbf{M} \gets |\operatorname{RFFT}(\mathbf{w})|_{1:}$
       \hfill $\triangleright$ discard DC bin
\STATE $k^* \gets \arg\max_k M_k$
\STATE $A \gets 4\, M_{k^*} / n$
       \hfill $\triangleright$ Hann-corrected amplitude

\STATE \textit{--- Classify (first match wins) ---}
\STATE $n_{\text{close}} \gets \#\{i : s_i \geq 1 - \tau\}$
       \hfill $\triangleright$ (or $s_i \leq \tau$ for norm)
\IF{$n_{\text{close}} \geq \rho\, n$}
    \STATE \textbf{return} $\textsc{FixedPoint}$
\ENDIF
\STATE \hfill $\triangleright$ peak-to-peak $= 2A \geq \tau$; at least 2 full cycles in window
\IF{$A \geq \tau/2$ \textbf{ and } $(k^*+1)/n \geq 2/n$}
    \STATE \textbf{return} $\textsc{Orbit}\!\left(\text{freq} = (k^*+1)/n,\; \text{amp} = A\right)$
\ENDIF
\STATE $g \gets \hat{a}$
       \hfill $\triangleright$ linear-fit slope; negate for norm series
\STATE \hfill $\triangleright$ sim increases by ${\geq}\,\tau$ over the full window
\IF{$g > \tau / n$}
    \STATE \textbf{return} $\textsc{Slider}(g)$
\ENDIF
\STATE \textbf{return} $\textsc{Unknown}$
\end{algorithmic}
\end{algorithm}

\begin{figure}
    \centering
    \includegraphics[width=\linewidth]{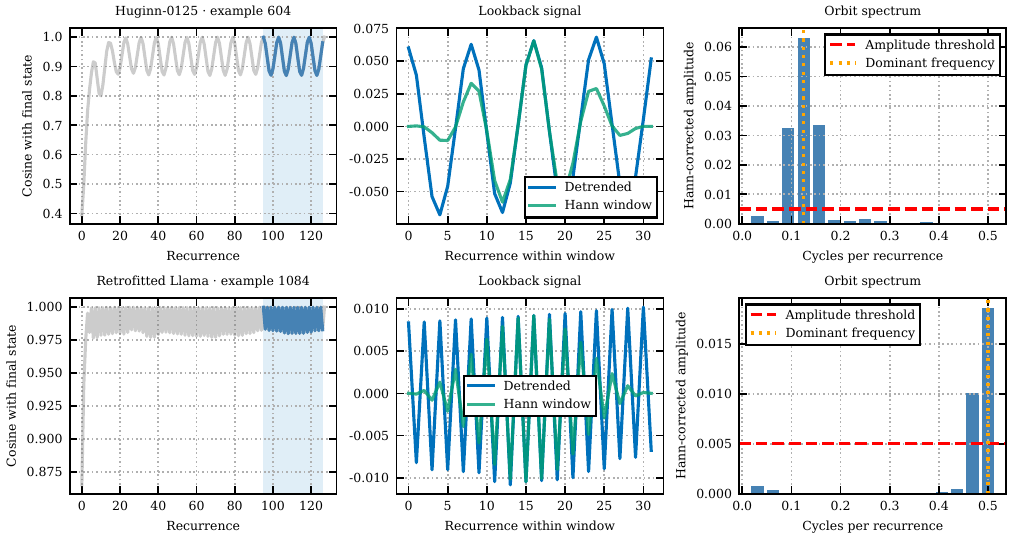}
    \caption{Visualizing the component parts of the Orbit detection algorithm of \Cref{alg:orbit_slider}. The input sequence (cosine similarities for the residual streams of a given token and layer and successive recursions, as compared to their final residual stream) is visualized in the leftmost column. The center column visualizes the effect of windowing and de-trending, and the rightmost column shows the FFT magnitudes. The top row visualizes the detected Orbit with the largest amplitude for the \raven{} model (which occurs with the ``Long Persona'' prompt) and the bottom row visualizes the largest amplitude for the \rllama{} model (which occurs with no system prompt). We note the appearance of complex, multi-frequency oscillation in this latter case.}
    \label{fig:orbit_algorithm_examples}
\end{figure}

\begin{table}
\centering
\footnotesize
\caption{Percentage of tokens exhibiting each behavior type, by model and system prompt.}
\label{tab:behavior_tokens}
\begin{tabular}{llrrrr}
\toprule
Model & System Prompt & Non-Fixed-Point \% & Orbit \% & Slider \% & Unknown \% \\
\midrule
Huginn & Long Persona & 0.14 & 0.13 & 0.00 & 0.01 \\
Huginn & Long Persona (Padded) & 0.05 & 0.02 & 0.02 & 0.01 \\
Huginn & No System Prompt & 0.02 & 0.01 & 0.01 & 0.00 \\
Huginn & Short Math & 0.01 & 0.00 & 0.00 & 0.00 \\
Retrofitted-Llama & Long Persona & 0.00 & 0.00 & 0.00 & 0.00 \\
Retrofitted-Llama & Long Persona (Padded) & 0.00 & 0.00 & 0.00 & 0.00 \\
Retrofitted-Llama & No System Prompt & 0.00 & 0.00 & 0.00 & 0.00 \\
Retrofitted-Llama & Short Math & 0.00 & 0.00 & 0.00 & 0.00 \\
\bottomrule
\end{tabular}
\end{table}

\begin{table}
\centering
\caption{Percentage of examples exhibiting each behavior type at least once on any question token, by model and system prompt.}
\footnotesize
\label{tab:behavior_existence}
\begin{tabular}{llrrrr}
\toprule
Model & System Prompt & Non-Fixed-Point \% & Orbit \% & Slider \% & Unknown \% \\
\midrule
Huginn & Long Persona & 2.81 & 2.50 & 0.15 & 1.06 \\
Huginn & Long Persona (Padded) & 0.83 & 0.23 & 0.61 & 0.15 \\
Huginn & No System Prompt & 0.76 & 0.53 & 0.15 & 0.08 \\
Huginn & Short Math & 0.45 & 0.38 & 0.08 & 0.00 \\
Retrofitted-Llama & Long Persona & 0.00 & 0.00 & 0.00 & 0.00 \\
Retrofitted-Llama & Long Persona (Padded) & 0.00 & 0.00 & 0.00 & 0.00 \\
Retrofitted-Llama & No System Prompt & 0.08 & 0.08 & 0.00 & 0.00 \\
Retrofitted-Llama & Short Math & 0.00 & 0.00 & 0.00 & 0.00 \\
\bottomrule
\end{tabular}
\end{table}

\FloatBarrier
\subsection{Do Intermediate Layers Exhibit Non-Fixed-Point Behavior?}

\citet{geiping2025scaling} observe orbits and sliders in the latent states after each application of the entire recurrent block. Here we investigate what occurs in the latent states of the intermediate layers within the recurrent block.

We first extend Figure 16 of \citet{geiping2025scaling} (visualizing latent trajectories on a math prompt), visualizing also the trajectories of the intermediate layer residual streams: this can be found in \Cref{fig:intermediate_nfp_behavior}. This demonstrates that in this particular case, where Orbits occur, they also occur in the intermediate layers. This implies -- viewed in realized depth -- latent trajectories exhibiting multi-scale cyclic behavior.

We investigate this across all GSM8k test examples in \Cref{fig:nfp_behavior_conditional_probability} by plotting the conditional probability of observing each behavior on a given token for any layer, given an observation of another behavior on that same token. We discover that orbits and sliders do not co-occur across looped layers, but that orbits and sliders both frequently co-occur with fixed point behavior. ``Unknown'' behavior often co-occurs with orbits, which we suggest may be due to mis-classification of the behavior algorithm.

\begin{figure}[ht]
    \centering
    \includegraphics[width=0.6\linewidth]{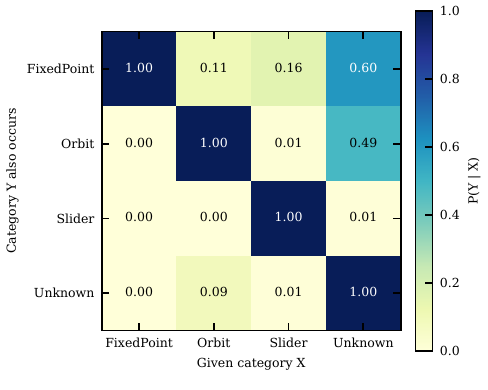}
    \caption{Conditional probabilities of co-occurrence for the different limiting behaviors.}
    \label{fig:nfp_behavior_conditional_probability}
\end{figure}

\begin{figure}
    \centering
    \includegraphics[width=\linewidth]{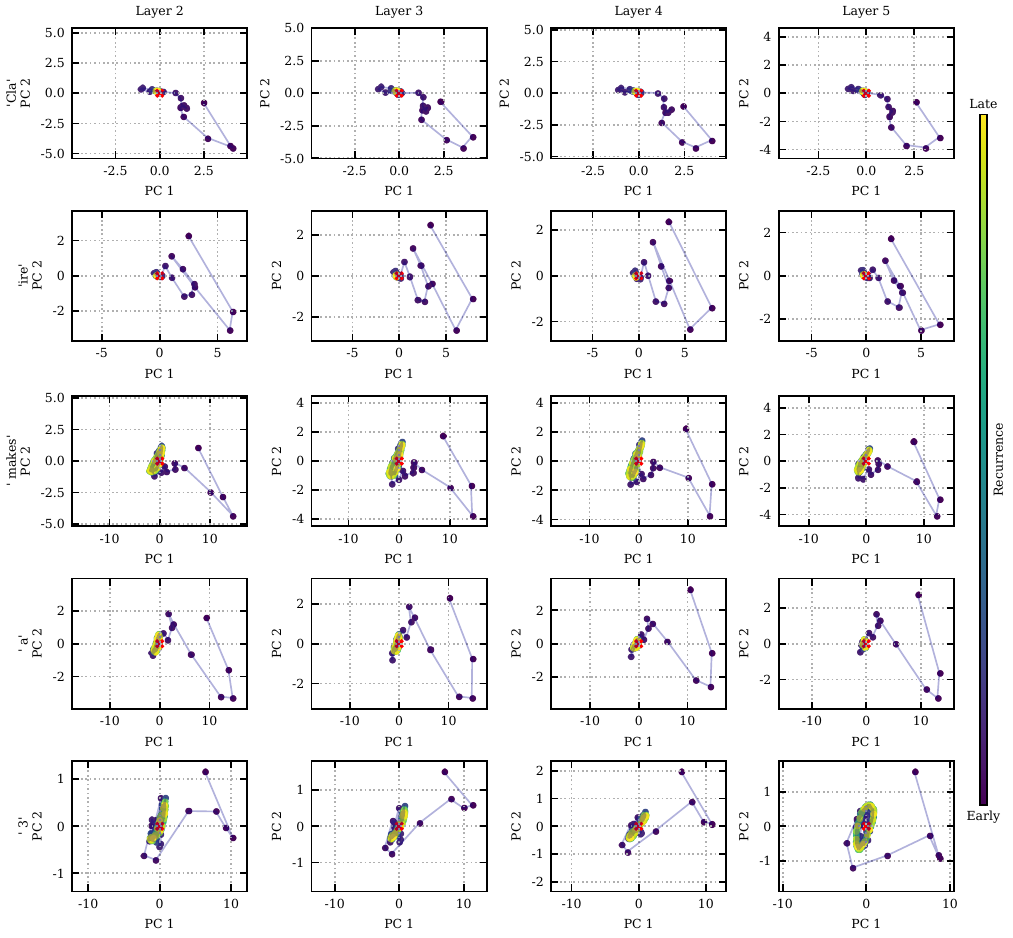}
    \caption{PCA trajectories in the intermediate layers of \raven{}: this reproduces the leftmost column of Fig. 16 in \citet{geiping2025scaling} (the first two principal components) and additionally plots the latent trajectories for the intermediate layers in the recurrent block.}
    \label{fig:intermediate_nfp_behavior}
\end{figure}

\FloatBarrier
\subsection{How Does Non-Fixed-Point Behavior Impact Stages of Inference?}

To complete the link between non-fixed-point behavior and our work we additionally investigate how this behavior impacts the observed stages of inference.

To attempt to isolate the ``worst case'' scenario for stages of inference stability, we plot stages of inference for the GSM8k test prompt that exhibits the greatest orbit amplitude. We first plot in realized depth the extended stages of inference metrics used throughout the paper, see \Cref{fig:orbits_sliders_huginn_stability_in_recurrence}. This demonstrates that sink rates show some variability with the orbiting behavior, but the other metrics remain broadly consistent. We plot also the same metrics in block depth (showing only the final 64 loops) in \Cref{fig:orbits_sliders_huginn_stability_in_depth}: this demonstrates clearly that the stages of inference remain very consistent despite the orbiting behavior.

\begin{figure}[ht]
    \includegraphics[width=\linewidth]{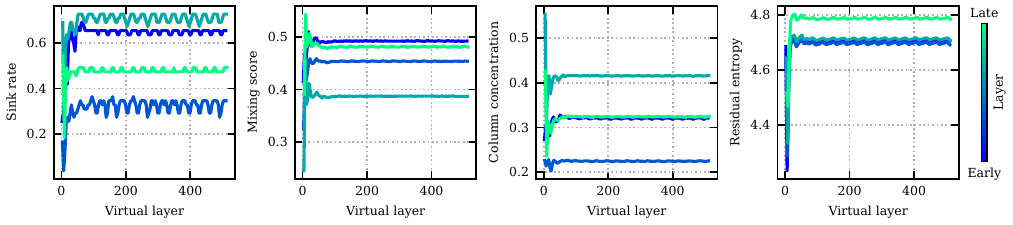}
    \caption{Stages of inference metrics (sink rate, mixing score and colsum concentration) for \raven{} across 128 recurrences, for the GSM8k test prompt that exhibited the largest orbit amplitude. Visualized in realized depth.}
    \label{fig:orbits_sliders_huginn_stability_in_recurrence}
\end{figure}

\begin{figure}[ht]
    \centering
    \includegraphics[width=\linewidth]{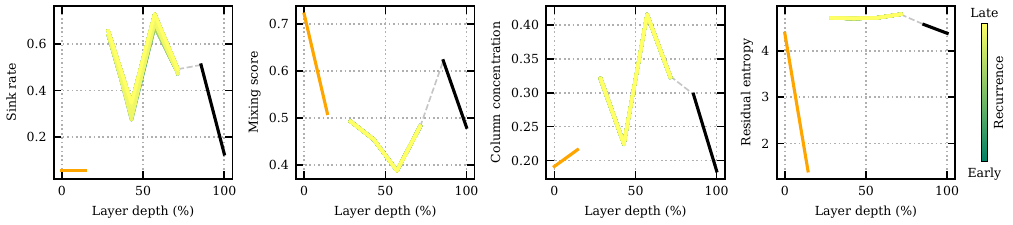}
    \caption{Stages of inference metrics (sink rate, mixing score and colsum concentration) for \raven{} across 128 recurrences, for the GSM8k test prompt that exhibited the largest orbit amplitude: only the final 64 recurrences are visualized, to isolate the impact of the orbit. Visualized in percentage block depth.}
    \label{fig:orbits_sliders_huginn_stability_in_depth}
\end{figure}

We plot the same for the largest amplitude orbit on the \rllama{} model in \Cref{fig:orbits_sliders_llama_stability_in_depth}, demonstrating that the same behavior holds.

\begin{figure}[ht]
    \centering
    \includegraphics[width=\linewidth]{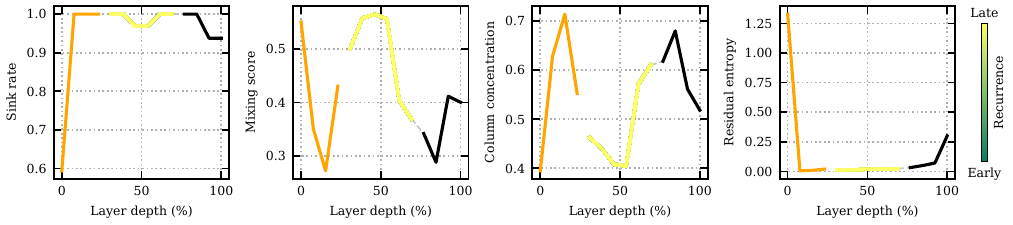}
    \caption{Stages of inference metrics (sink rate, mixing score and colsum concentration) for \rllama{} across 128 recurrences, for the GSM8k test prompt that exhibited the largest orbit amplitude: only the final 64 recurrences are visualized, to isolate the impact of the orbit. Visualized in percentage block depth.}
    \label{fig:orbits_sliders_llama_stability_in_depth}
\end{figure}

%% file: sections/appendices/additional_fp_results.tex
\FloatBarrier

In this appendix we include additional plots to validate and provide further context to our fixed point claims in \Cref{sec:fixed_points}.

\begin{figure}[htp]
    \centering
    \includegraphics[width=\linewidth]{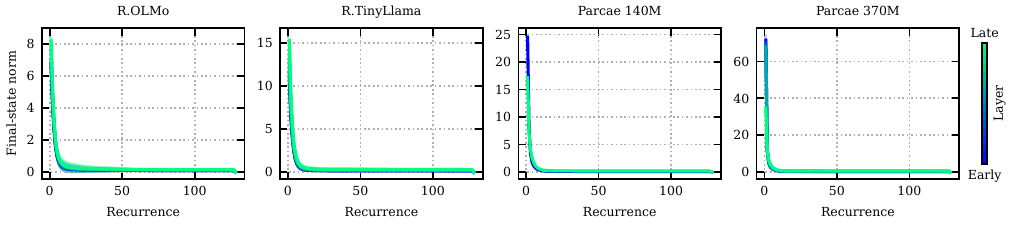}
    \caption{Complement to \Cref{fig:fixed_point_norms_main} with additional models, visualizing the norm of the difference from each layer to an ``approximate fixed point'' (the 128th recurrence).}\label{fig:fixed_point_norms_appendix}
\end{figure}

\begin{figure}[htp]
    \centering
    \includegraphics[width=\linewidth]{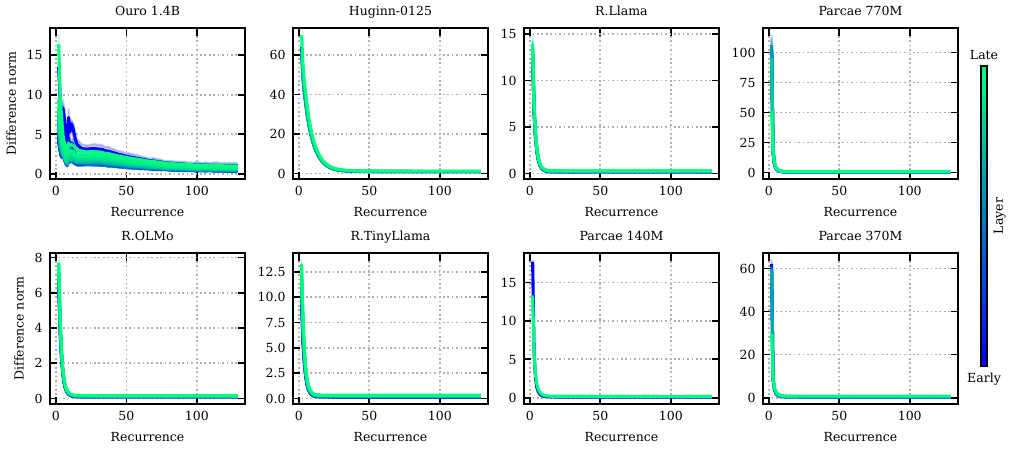}
    \caption{Direct validation of the premise of \Cref{prop:similar_attention_patterns}, visualizing the norm of the difference between successive residual streams for each layer $l$, $| \vh_{l, r} - \vh_{l, r-1} |_2$. This successive difference continues to decrease to a low level even for \ouro{}, which does not reach a fixed point.}\label{fig:successive_recurrence_residual_diff_norm_all}
\end{figure}
\begin{figure}[htp]
    \centering
    \includegraphics[width=\linewidth]{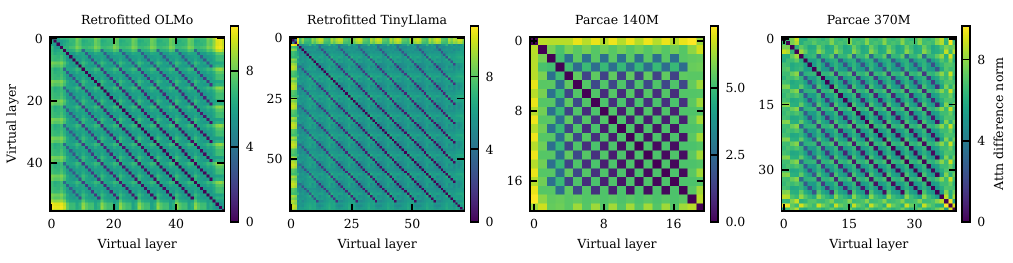}
    \caption{Complement to \Cref{fig:pairwise_attn_pattern_diff_norm_main} with additional models, visualizing Frobenius norms between attention patterns at all virtual layer pairs (8 loops), averaged across the batch and head dimensions.}\label{fig:pairwise_attn_pattern_diff_norm_appendix}
\end{figure}

\begin{figure}[htp]
    \centering
    \includegraphics[width=\linewidth]{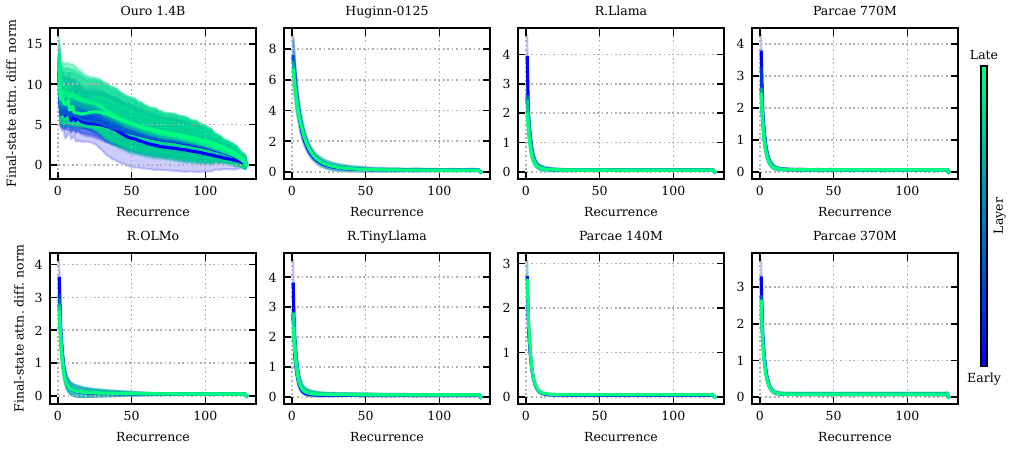}
    \caption{Attention pattern stability across recurrences. Visualises the Frobenius norm of the difference from the attention patterns at each layer to an ``approximate fixed point'' (the 128th recurrence): for each layer $l$, $| \mA_{l, r} - \mA_{l, 128} |_F$, averaged over batch and head dimensions.}\label{fig:fixed_point_attention_differences_all}
\end{figure}

\FloatBarrier
\subsection{Additional Architecture Results}

In \Cref{fig:stability_initialised_to_final_state_residual_diff_norm} we present the norm of the residual difference, complementing the cosine similarity results of \Cref{fig:summarised_stability}. In \Cref{fig:stability_initialised_to_final_state_residual_diff_norm_all_layers} we further check that our stability conclusions hold for varying layer counts: we find that broadly they do, with the partial exception that in the case of pre-norm with input injection, some seeds with 16 layers exhibit oscillation rather than a true fixed point.

We note that the continued residual growth of pre-norm models without input injection initially seems at odds with the pretrained convergence results of \citet[Appendix B.4.1]{fumero2026navigating} who analyze the convergence of pretrained models, however their experiments include the final layernorm immediately pre-softmax in their recurrent block, explaining the bounded norm.

\begin{figure}[htp]
    \centering
    \includegraphics[width=\linewidth]{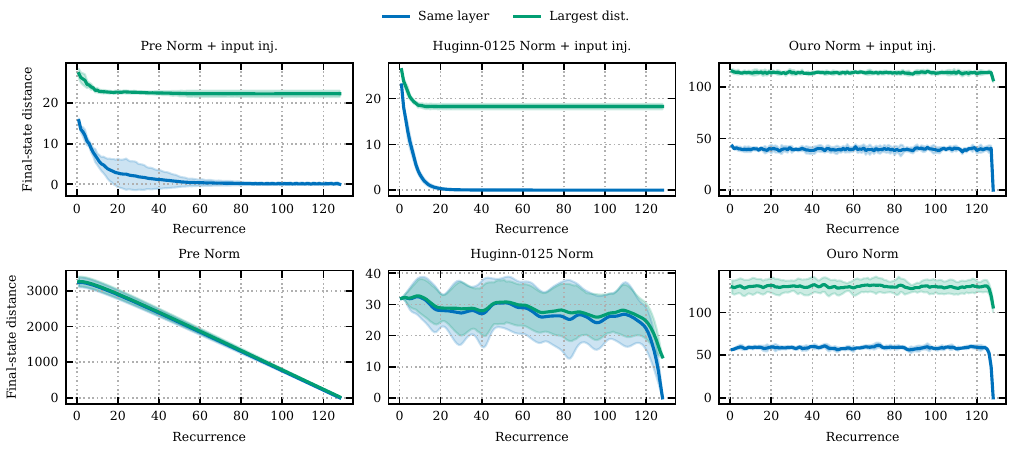}
    \caption{Complement to \Cref{fig:summarised_stability}, showing instead the difference norm to fixed point per layer $l$ (in blue) $| \vh_{l, r} - \vh_{l, 128} |_2$ and also to the layer with the most distant fixed point $\tilde{l}$ (in green) $| \vh_{l, r} - \vh_{\tilde{l}, 128} |_2$, in order to potentially capture ``degenerate fixed points'' which would be constant across layers. As in the main body of the paper, averaged over 5 seeds.}\label{fig:stability_initialised_to_final_state_residual_diff_norm}
\end{figure}

\begin{figure}[htp]
    \centering
    \includegraphics[width=\linewidth]{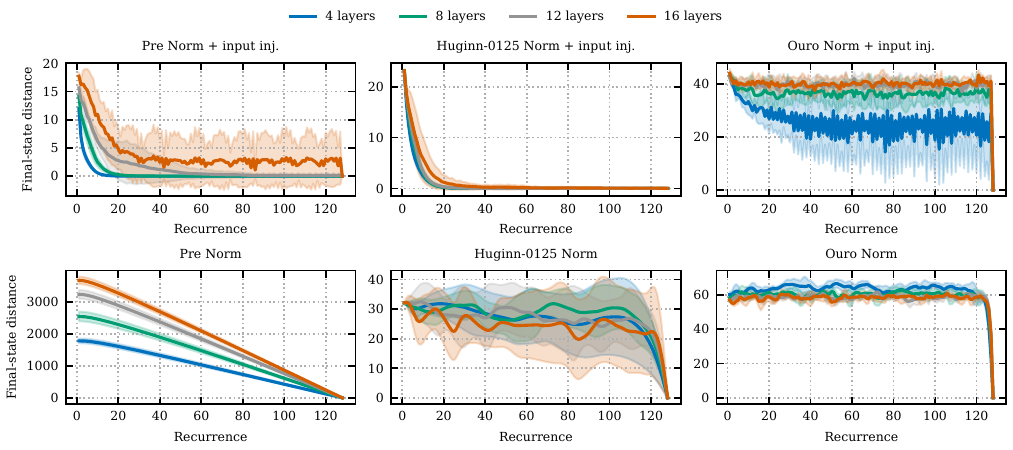}
    \caption{Complement to \Cref{fig:summarised_stability}, showing the impact of different numbers of layers in the recurrent block. Broadly the same trends hold across the architectural choices tested, but for more layers some of the prenorm seeds begin to exhibit oscillation, explaining the larger deviations in the top left plot. As in the main body of the paper, averaged over 5 seeds.}\label{fig:stability_initialised_to_final_state_residual_diff_norm_all_layers}
\end{figure}

%% file: sections/appendices/additional_trajectory_results.tex
To provide additional intuition on the latent space trajectories followed by these models and support the qualitative conclusions drawn from \Cref{fig:retro_llama_trajectory}, in this section we plot the latent trajectories at multiple token positions (beginning, $\sim$ middle, end) and for multiple different prompts (randomly selected from the GSM8k test set) for the main models presented in the paper: 

\begin{figure}[htp]
    \centering
    \includegraphics[width=0.9\linewidth]{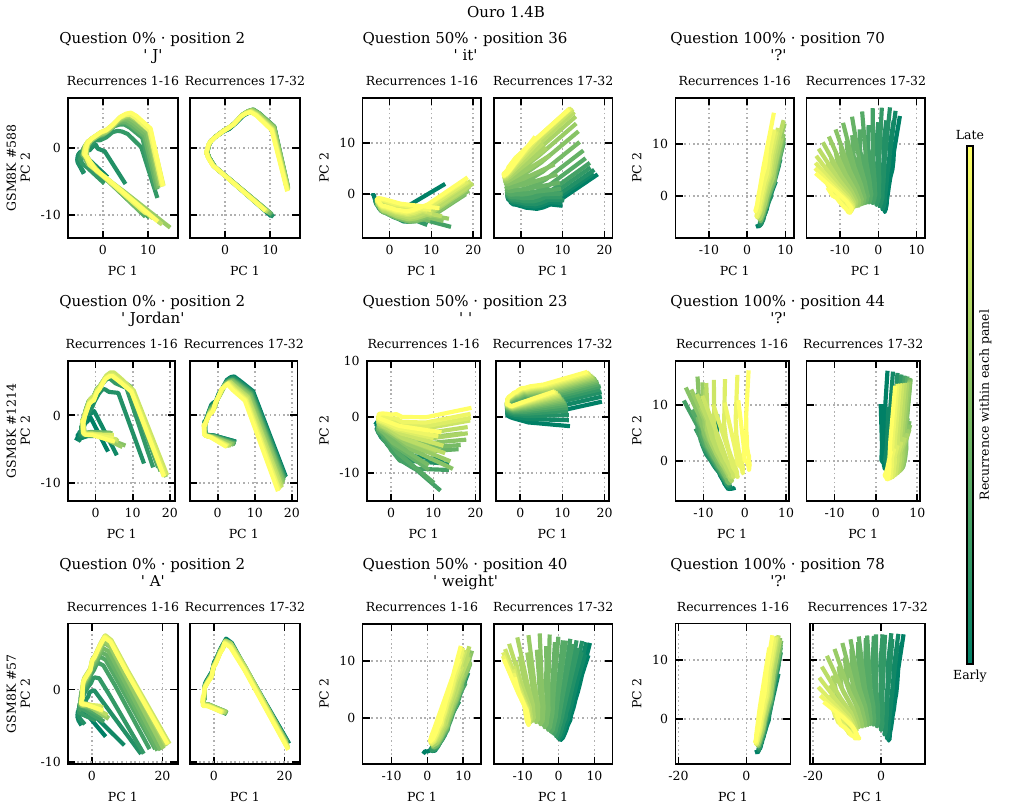}
    \caption{Example latent trajectories at different GSM8k prompts and different token positions for \ouro{}.}\label{fig:pca_gsm8k_grid_Ouro}
\end{figure}

\begin{figure}[htp]
    \centering
    \includegraphics[width=0.9\linewidth]{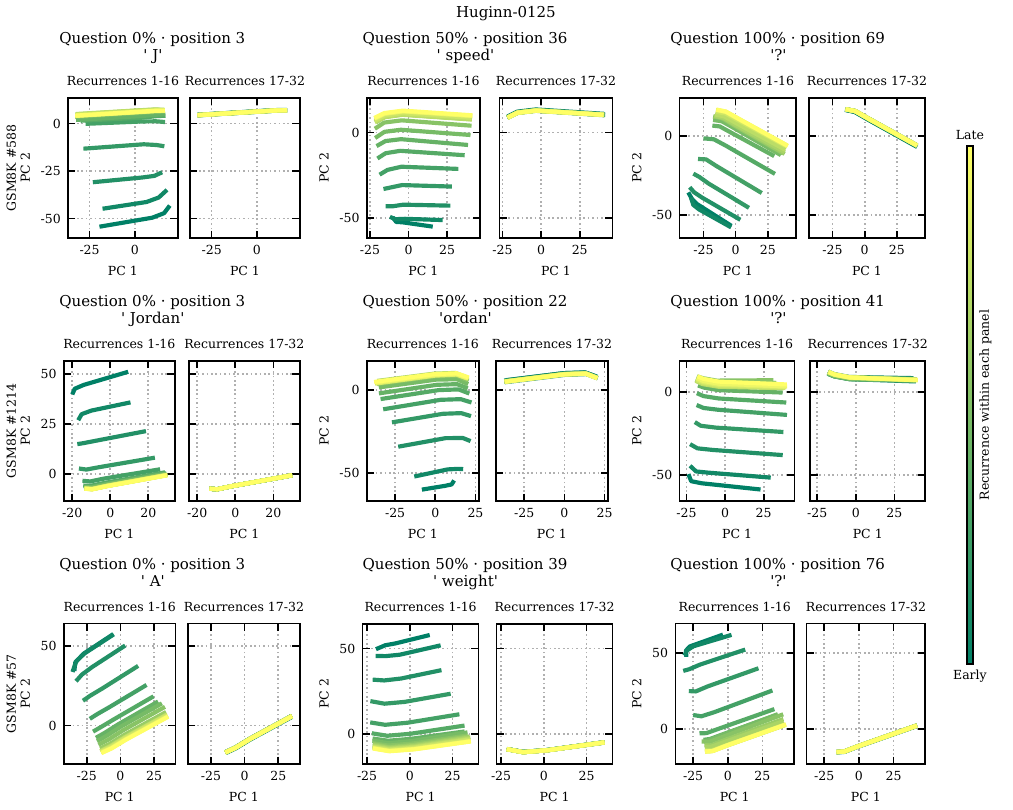}
    \caption{Example latent trajectories at different GSM8k prompts and different token positions for \raven{}.}\label{fig:pca_gsm8k_grid_Huginn}
\end{figure}

\begin{figure}[htp]
    \centering
    \includegraphics[width=0.9\linewidth]{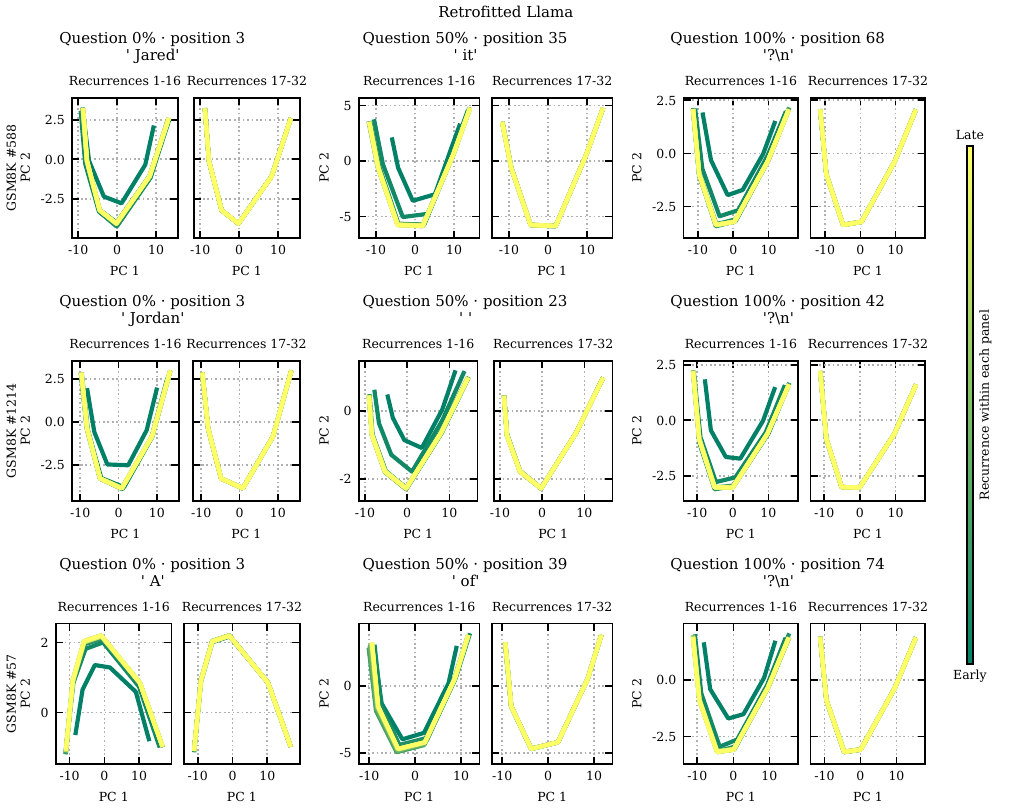}
    \caption{Example latent trajectories at different GSM8k prompts and different token positions for \rllama{}.}\label{fig:pca_gsm8k_grid_Retrofitted_Llama}
\end{figure}

\begin{figure}[htp]
    \centering
    \includegraphics[width=0.9\linewidth]{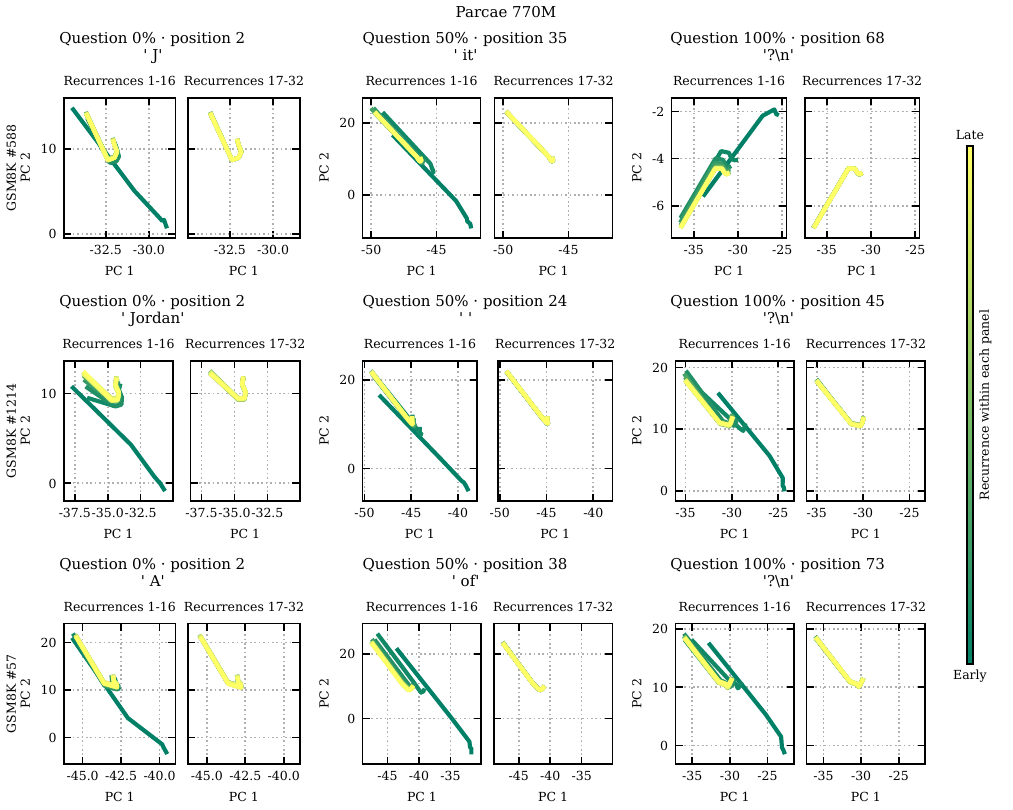}
    \caption{Example latent trajectories at different GSM8k prompts and different token positions for Parcae.}\label{fig:pca_gsm8k_grid_Parcae}
\end{figure}

%% file: sections/appendices/additional_soi_results.tex
In this appendix we explore in more detail the mixing behavior presented in the main body via additional metrics, and considering a wider range of models.

\subsection{Prediction and Suppression Neuron Classification}

As our work is concerned largely with how the functionality of layers changes throughout depth as the residual stream is iteratively updated, our primary concern is with \emph{input-dependent} measures of stages of inference: \mainmixing{} as presented in the main text of the paper is one such input-dependent metric, and the later sections of this appendix will introduce and present results for a wider range of such metrics.

However, to further link our work to that of \citet{lad2024remarkable}, we also present results for the fraction of \emph{prediction and suppression} neurons \citep{gurnee2024universal} in successive layers of both feedforward and \acp{LT}. We highlight however that \emph{these are unable to change with successive recurrences}, and are thus secondary to our focus. \Cref{fig:feedforward_prediction_neurons} presents results for a selection of feedforward models used throughout the paper, and \Cref{fig:looped_prediction_neurons} presents results for a selection of looped models used throughout the paper. Similar to our results elsewhere, we find that the looped model stages of inference in these metrics tend to mirror those of feedforward models. We highlight that the existence of the prediction and suppression neurons in the repeated block of \ouro{} in the leftmost figure of \Cref{fig:looped_prediction_neurons} supports the results presented in \Cref{fig:logitlens_summary}, wherein initial layers in later loops must ``decode'' back out of this token space. In other models, these neurons tend to be delegated to the coda, which runs only once.

\begin{figure*}[ht]
    \centering    
    \includegraphics[width=\linewidth]{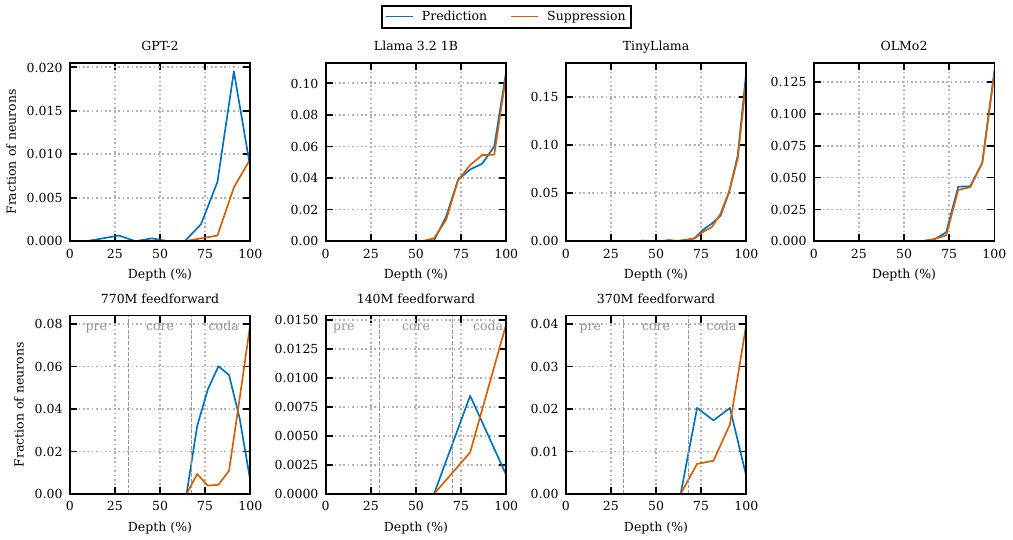}
    \caption{\textbf{Fraction of prediction and suppression neurons} in a selection of feedforward models used throughout the paper.}
    \label{fig:feedforward_prediction_neurons}
\end{figure*}

\begin{figure*}[ht]
    \centering    
    \includegraphics[width=\linewidth]{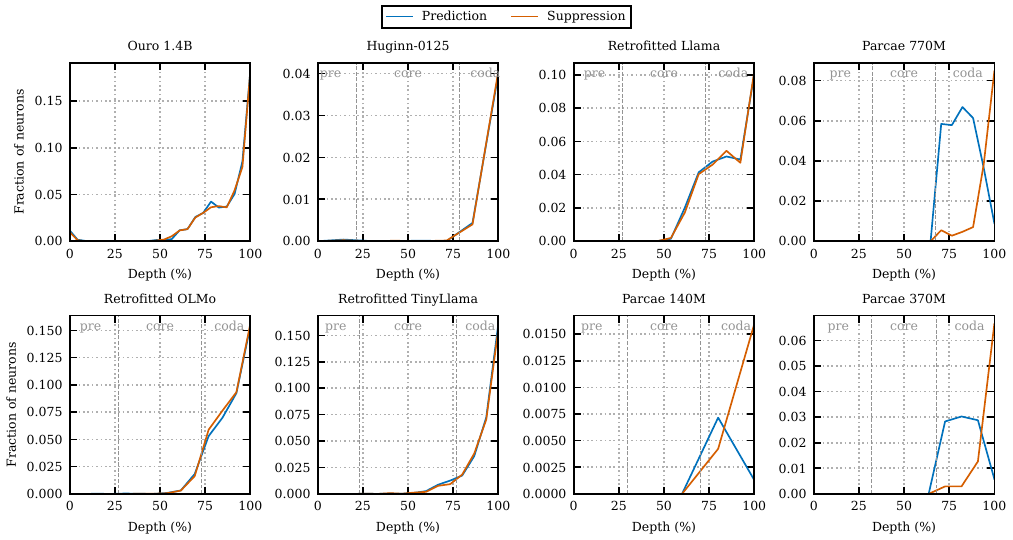}
    \caption{\textbf{Fraction of prediction and suppression neurons} in a selection of looped models used throughout the paper.}
    \label{fig:looped_prediction_neurons}
\end{figure*}

\FloatBarrier
\subsection{Cyclic Attention Behavior}\label{appen:additional_cyclic_attention_behavior}

One well-studied phenomenon by which Transformers drastically \emph{reduce} the mixing in a given layer is that of the \emph{attention sink} \citep{xiao2023efficient,barbero2025llms}, whereby the layer focuses the majority of the attention ``weight'' onto the first token in the sequence; often this is the BOS (beginning of sequence) token, and thus completely uninformative. To measure this behavior, we adopt the \textbf{attention sink score} and \textbf{sink rate} of \citet{gu2024attention}: for token position $k$ and sequence length $T$, the sink score at layer $\ell$ and head $h$ is defined as:
\begin{equation}
    \text{sink-score}_k^{(\ell,h)} = \frac{1}{T}\sum_{t=0}^{T-1}A_{tk}^{(\ell,h)},
\end{equation}
where $A_{tk}^{(\ell,h)}$ corresponds to the realized attention pattern of \Cref{eq:attn_in_a} at layer $\ell$, head $h$. The sink rate is then defined as the fraction of heads for which the sink score lies above a certain threshold:
\begin{equation}
    \text{sink-rate}^{(\ell)}_k = \frac{1}{H}\sum_{h=1}^H\mathbb{I}\!\left(\text{sink-score}_k^{(\ell,h)}\geq \tau\right),
\end{equation}
where following related work we define a threshold of $\tau=0.3$, and $\mathbb{I}$ denotes the indicator function.

We also adapt the \emph{Mixing score} of \citet{queipo2025attention}. This was initially defined as follows: for sequence length $T$, layer $\ell$ and head $h$ the Mixing Score is the average row entropy of the attention patterns:
\begin{equation}
    \text{mixing-score}^{(\ell,h)} = \frac{1}{T} \sum_{i=1}^T H(A_{i, :}^{(\ell, h)}),
\end{equation}
and there exists a trivial extension of this that normalizes by $\log{T}$ to achieve a value in $[0,1]$. However, this has some unintuitive properties, as illustrated in the ``Uniform Causal'' example in \Cref{fig:mixing_metric_investigation}: for a uniform (``maximally mixed'') causal attention pattern, the Mixing Score still doesn't take a value of 1 for two reasons:
\begin{enumerate}
    \item The Shannon entropy of row $i$ is normalized by $\log{T}$ but due to the causal mask the maximum value it can take is $\log{i}$. Therefore, all rows other than the last are constrained to a normalized entropy below 1.
    \item Due to the causal mask, the entropy of the first row will always be $H(A_{1, :}) = - 1 \log{1} = 0$.
\end{enumerate}

We address these two issues in our modified \mainmixing{} of \Cref{eq:mixing_score} by omitting the first row and normalizing each row by $\log{i}$: this allows the score to truly take values in the entire range of $[0,1]$, and means that uniform causal attention patterns have a \mainmixing{} of 1.

We also use the ColSum Concentration from \citet{queipo2025attention}. For this we first define the Column sum $c_j = \sum_i A_{ij}$ to capture how much attention mass is received by token $j$. We normalize this as $\hat{c}_j = c_j / T$ to obtain a probability distribution, noting that $\sum_{i,j} A_{ij} = T$ since $A$ is row-stochastic. From this distribution, we define the \textbf{ColSum Concentration} via its normalized entropy as
\begin{equation}
C = 1 - H_{\text{col}} \in [0,1] = 1 + \frac{1}{\log T} \sum_j \hat{c}_j \log \hat{c}_j.
\end{equation}
Large values of $C$ indicate a high \emph{concentration} of attention mass: few columns receive most of the mass. We note therefore that this metric captures the well-studied \emph{attention sink} \citep{xiao2023efficient,barbero2025llms} behavior, but generalizes to capture concentration over any token position. This property is useful for our investigation as not all models studied herein exhibit sinks on the first token; in particular OLMo-2 frequently concentrates attention mass on punctuation, echoing a result in \citet{sandoval2025using}.

In the plots that follow we average sink rates, Mixing Scores and ColSum concentrations over all the heads in a layer, and over all input sequences.

We plot all of these metrics as they measure subtly different things. The sink rate measures the specific occurrence of the first-token attention sink phenomena but misses split sinks and other ``sharp'' attention patterns; the ColSum concentration addresses this and captures split sinks and sinks on other tokens but still misses other ``sharp'' attention patterns like identity and previous token heads; the Mixing Score loses fidelity on specific sinking behavior but captures a wide range of ``sharp'' vs ``sparse'' attention patterns. See \Cref{fig:mixing_metric_investigation} for a visualization of some of these differences.

\begin{figure*}[ht]
    \centering    
    \includegraphics[width=\linewidth]{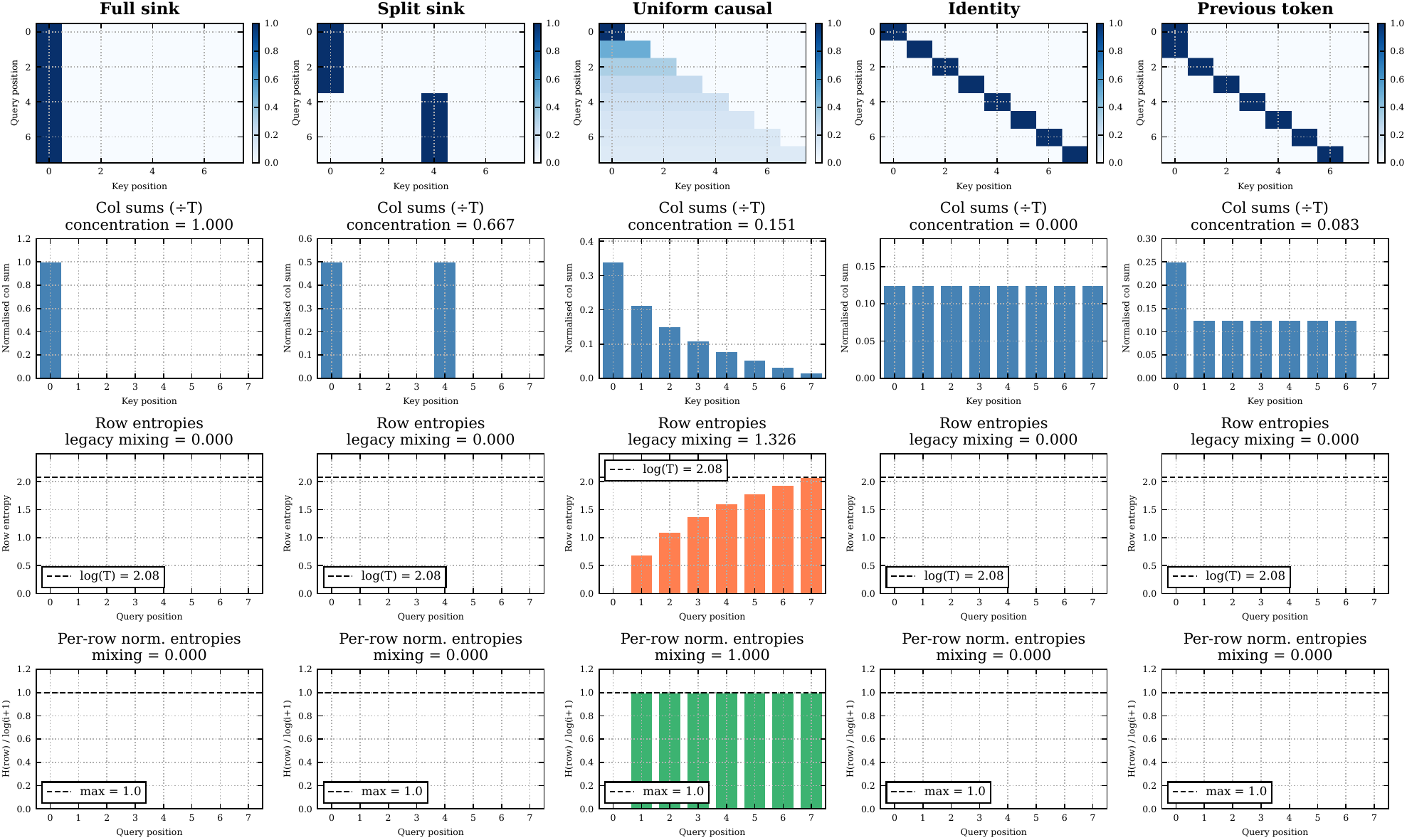}
    \caption{\textbf{Stages of Inference metrics on representative attention patterns:} ColSum Concentration, (legacy) Mixing Score \citep{queipo2025attention} and our adapted Mixing Score.}
    \label{fig:mixing_metric_investigation}
\end{figure*}

We first supplement the mixing / attention pattern conclusions drawn from \Cref{fig:summarised_recurrent_block_stages} with both additional models and additional attention metrics: \ouro{} (\Cref{fig:additional_attention_metrics_ouro}), \raven{} (\Cref{fig:additional_attention_metrics_raven}), \rllama{} (\Cref{fig:additional_attention_metrics_rllama}), \rolmo{} (\Cref{fig:additional_attention_metrics_rolmo}), \rtllama{} (\Cref{fig:additional_attention_metrics_rtllama}), Parcae 370M (\Cref{fig:additional_attention_metrics_parcae_370}), Parcae 770M (\Cref{fig:additional_attention_metrics_parcae_770}). These figures demonstrate that for \textbf{sink rates and mixing scores}, the cyclic behavior and lack of significant depth-wise changes hold both for the other investigated models, and the additional stages of inference metrics.

\begin{figure}[htp]
    \centering
    \includegraphics[width=\linewidth]{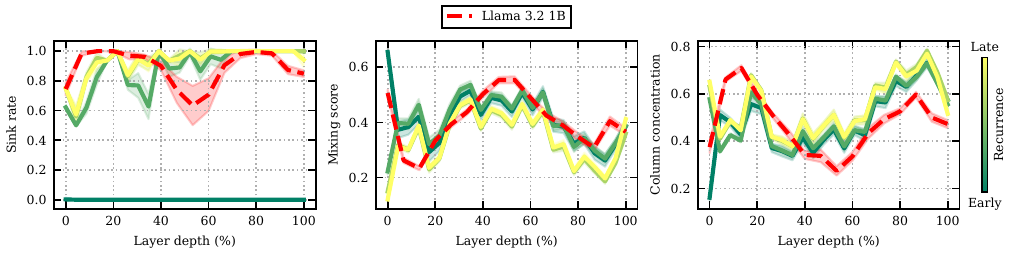}
    \caption{Complement to \Cref{fig:summarised_recurrent_block_stages}, visualising additional attention metrics for \ouro{}.}\label{fig:additional_attention_metrics_ouro}
\end{figure}

\begin{figure}[htp]
    \centering
    \includegraphics[width=\linewidth]{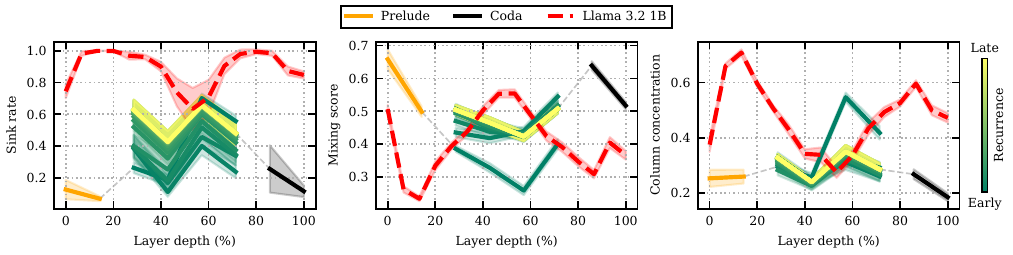}
    \caption{Complement to \Cref{fig:summarised_recurrent_block_stages}, visualising additional attention metrics for \raven{}.}\label{fig:additional_attention_metrics_raven}
\end{figure}

\begin{figure}[htp]
    \centering
    \includegraphics[width=\linewidth]{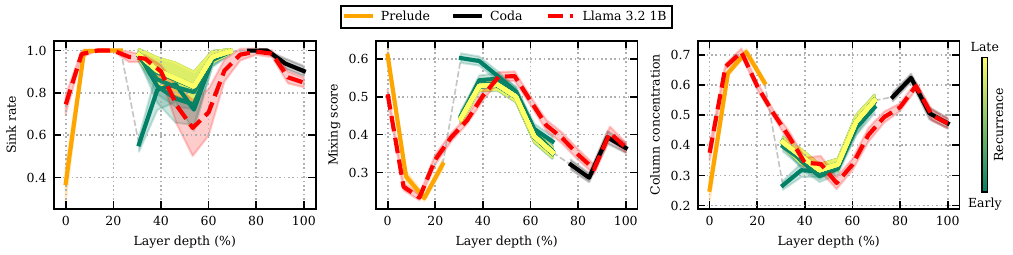}
    \caption{Complement to \Cref{fig:summarised_recurrent_block_stages}, visualising additional attention metrics for \rllama{}.}\label{fig:additional_attention_metrics_rllama}
\end{figure}

\begin{figure}[htp]
    \centering
    \includegraphics[width=\linewidth]{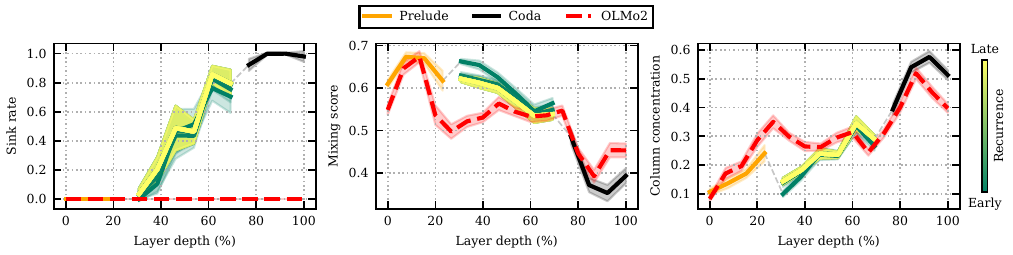}
    \caption{Complement to \Cref{fig:summarised_recurrent_block_stages}, visualising additional attention metrics for \rolmo{}.}\label{fig:additional_attention_metrics_rolmo}
\end{figure}

\begin{figure}[htp]
    \centering
    \includegraphics[width=\linewidth]{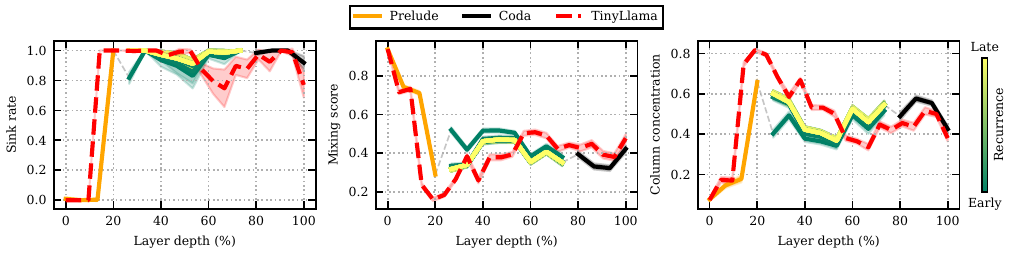}
    \caption{Complement to \Cref{fig:summarised_recurrent_block_stages}, visualising additional attention metrics for \rtllama{}{}.}\label{fig:additional_attention_metrics_rtllama}
\end{figure}

\begin{figure}[htp]
    \centering
    \includegraphics[width=\linewidth]{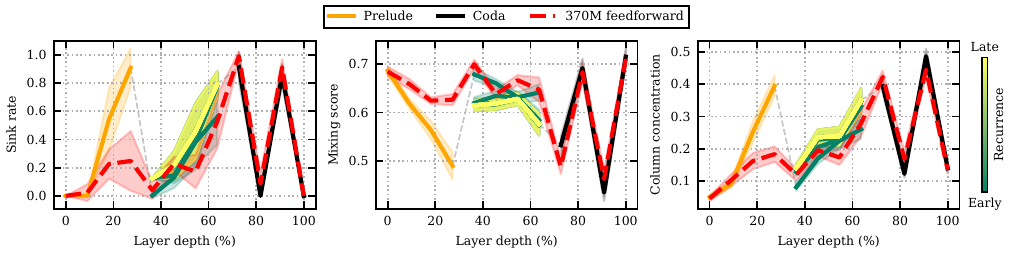}
    \caption{Complement to \Cref{fig:summarised_recurrent_block_stages}, visualising additional attention metrics for Parcae 370M.}\label{fig:additional_attention_metrics_parcae_370}
\end{figure}

\begin{figure}[htp]
    \centering
    \includegraphics[width=\linewidth]{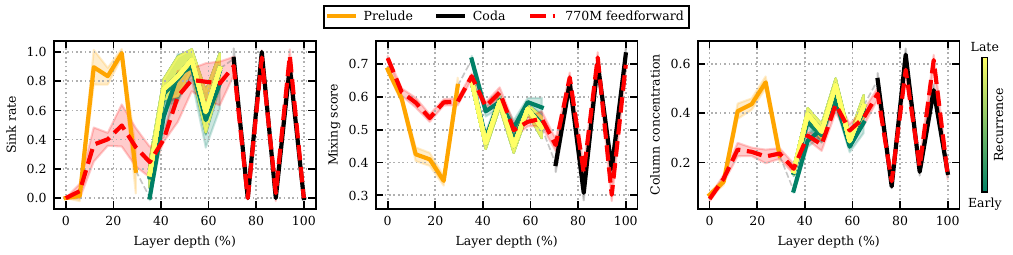}
    \caption{Complement to \Cref{fig:summarised_recurrent_block_stages}, visualising additional attention metrics for Parcae 770M.}\label{fig:additional_attention_metrics_parcae_770}
\end{figure}

Following \citet{skean2025layer,queipo2025attention} we additionally measure the compression of the residual stream $\mX$ via the matrix-based entropy $H(\mX)$. This is averaged over all input sequences.

The results for residual entropy, visualized in \Cref{fig:residual_entropy_all}, show almost the same story: \raven{} as before shows no recognizable stages of inference, while all the retrofitted and Parcae models broadly modulate residual stream entropy throughout recurrent depth similarly to their feedforward counterparts. However, \ouro{} demonstrates different behavior: the residual entropies both change significantly with successive recurrences and do not closely follow the feedforward behavior. We believe that the divergence from feedforward behavior may derive from the norm structure of this model: the residual stream is normalised after each recurrent block, as visualised in \Cref{fig:residual_norm_comparison}, thus periodically shutting down massive activations. This is not the case for the retrofitted or Parcae series of models, which lack this norm and show much closer alignment with feedforward stages of inference.

\begin{figure}[ht]
    \centering    
    \includegraphics[width=\linewidth]{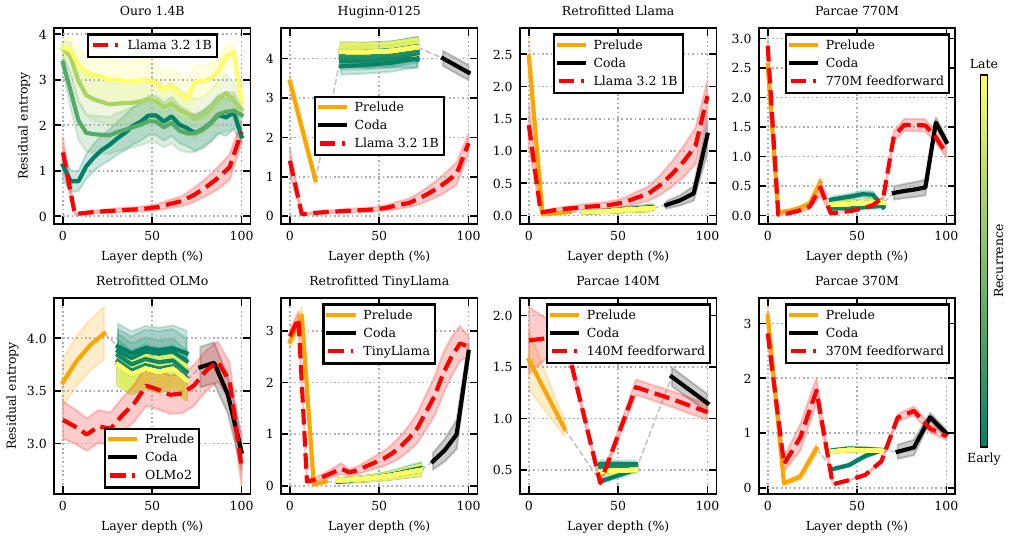}
    \caption{\textbf{Residual entropy with recurrent depth for all models explored in this paper}.}
    \label{fig:residual_entropy_all}
\end{figure}

\begin{figure}[ht]
    \centering    
    \includegraphics[width=\linewidth]{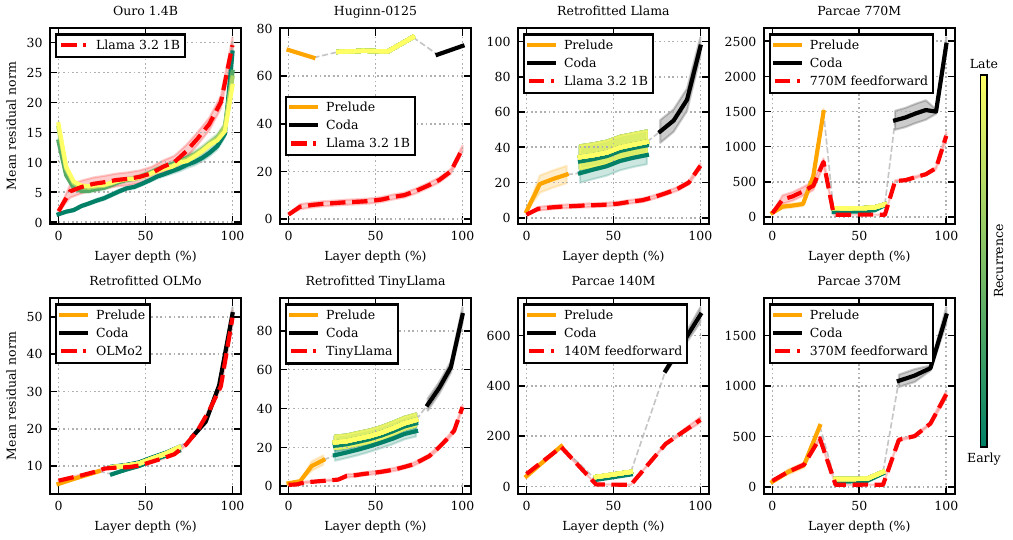}
    \caption{\textbf{Residual norm with recurrent depth for all models explored in this paper}. Mean taken across tokens and input sequences.}
    \label{fig:residual_norm_all}
\end{figure}

We plot stages of inference for \ouro{} 2.6B in \Cref{fig:ouro_2_recurrent_block_stages}. This model is interesting due to the training regime followed by \citet{zhu2025scaling}, which ``upcycles'' a 48 layer model from the 24 layer 1.4B parameter model. As a consequence, the first and second half of each recurrent block each independently align with the Llama feedforward stages of inference.

\begin{figure}[htp]
    \centering    
    \includegraphics[width=\linewidth]{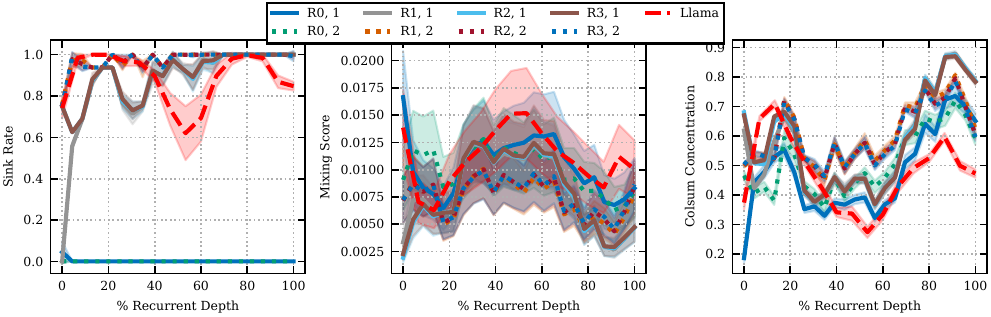}
    \caption{\textbf{Stages of inference for each recurrent loop in Ouro 2.6B.} For this model we separate out the first and second half of the recurrent block and overlay them, demonstrating that both halves have close alignment with the Llama feedforward stages of inference. We suggest that this likely arises due to the training regime of \citet{zhu2025scaling}, which first trains a single 24 layer 1.4B parameter model, and then ``upcycles'' this into a 48 layer 2.6B parameter model by duplicating these layers, consequently duplicating the stages of inference as well.}
    \label{fig:ouro_2_recurrent_block_stages}
\end{figure}

\FloatBarrier
\subsubsection{Mixing Score Correlations}\label{appen:mixing_score_correlations}

To quantify the similarity between the within-recurrence mixing-score profiles
of each looped model and its matched feedforward control
(\Cref{fig:summarised_recurrent_block_stages}), we correlate the two profiles across layer
depth. For recurrence $k$, the looped profile consists of the prelude, the $k$-th
application of the recurrent block, and the coda, in layer order. Ouro has no
prelude or coda, so its profile is a single recurrence. Where the looped model and
its control have different numbers of layers, we linearly interpolate the control's
profile at the looped model's normalised layer depths. We report two summaries, shown in \Cref{fig:mixing_correlation_scores}:
the correlation between the example-averaged profiles (as visualised elsewhere in the paper), and the mean (and standard
deviation) of correlations computed separately per example. The two agree
closely. Retrofitted Llama ($r = 0.89$, averaged over recurrences) and Parcae 770M ($r = 0.81$) closely
mirror their controls, and Ouro does so moderately ($r = 0.47$, $r=0.64$ on the first recurrence). Huginn-0125, as expected, has no positive correlation, with $r = -0.24$. Spearman's $\rho$, comparing only the rank order of layers, gives the
same qualitative picture.

\begin{figure*}[ht]
    \centering
    \includegraphics[width=\textwidth]{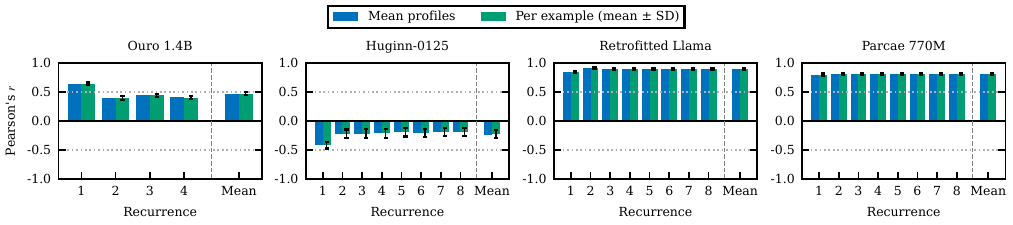}
    \includegraphics[width=\textwidth]{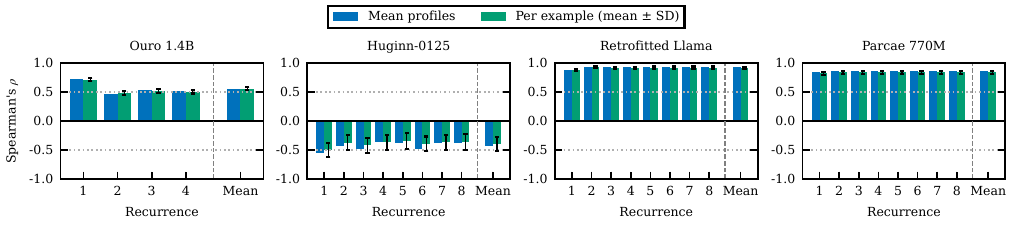}
    \caption{Correlation between the mixing-score profile of each recurrence and that
    of the matched feedforward control, measured by Pearson's $r$ (top) and
    Spearman's $\rho$ (bottom), on 256 GSM8K examples. Each looped profile comprises
    the prelude, one recurrence, and the coda. Controls are Llama 3.2 1B for Ouro,
    Huginn-0125 and Retrofitted Llama, and a size-matched one-recurrence Parcae
    model for Parcae 770M. Controls with a different layer count are interpolated
    at the looped model's layer depths. Blue bars correlate the example-averaged
    profiles. Green bars average per-example correlations, with whiskers showing
    one standard deviation across examples. The final bar in each panel is the mean
    over the recurrences shown. Ouro uses four recurrences. Other models show the
    first eight recurrences of a 32-recurrence rollout, whose coda follows the
    final recurrence.}
    \label{fig:mixing_correlation_scores}
\end{figure*}

\FloatBarrier

\subsubsection{\raven{} Negative Mixing Stages Result}\label{appen:negative_raven_result}

In \Cref{sec:stages_of_inference} we suggested that the lack of stages of inference in \raven{} is likely due to the normalization of the residual stream resulting in massive activations being unable to form: we demonstrate the impact of different norms by plotting residual stream magnitudes for a range of models in \Cref{fig:residual_norm_comparison}. This residual normalization means that \raven{} is unable to develop the growth in residual stream magnitude that \citet[Section 3.3]{queipo2025attention} note \emph{causes} compression behavior, leading to sink formation and mixing stages of inference.

\begin{figure}[htp]
    \centering    
    \includegraphics[width=\linewidth]{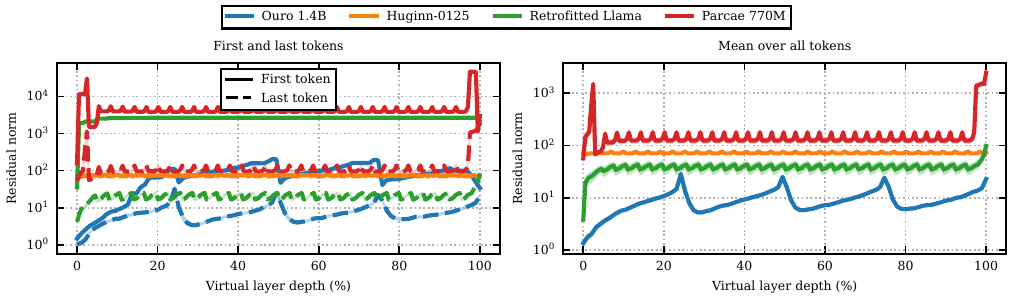}
    \caption{\textbf{Norms of the residual stream for a range of models,} demonstrating that \raven{} is unable to develop the activation magnitude changes required for stages of inference due to its repeated normalization of the residual stream.}
    \label{fig:residual_norm_comparison}
\end{figure}

We further support this hypothesis by ablating the massive activations from the \rllama{} model (which \emph{does} display stages of inference) via zeroing the output of the MLP in the second layer, which is responsible for its massive activations. In this setting, visualized in \Cref{fig:ablation_rllama}, we see that the model no longer exhibits stages of inference comparable to the feedforward model, suggesting that the presence of massive activations is required for stages of inference to emerge in looped models.

\begin{figure}[htp]
    \centering    
    \includegraphics[width=\linewidth]{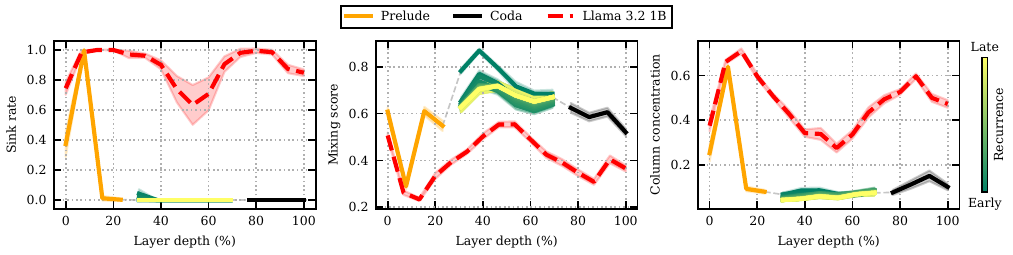}
    \caption{\textbf{Stages of inference for each recurrent loop in \rllama{} for which the massive activations have been ablated.} }
    \label{fig:ablation_rllama}
\end{figure}

\FloatBarrier

\subsubsection{Generalization to Non-Reasoning Tasks}\label{appen:non_reasoning_soi}

Throughout the main text of the paper, mixing experiments are conducted on the GSM8k dataset. In this appendix we verify that the attention pattern stages of inference we observe are not specific to this dataset, and also occur in a non-reasoning setting, for which we use the HellaSwag dataset \citep{zellers2019hellaswag}. We follow an identical experimental setup to the GSM8k experiments, running inference on 256 random examples from the test split.

To this end, we present in \Cref{fig:summarised_recurrent_block_stages_hellaswag} the equivalent of \Cref{fig:summarised_recurrent_block_stages}. We see that the mixing stages of inference on this dataset closely follow the GSM8k dataset. In the interest of space we do not reproduce every other figure in this paper on Hellaswag, however these figures are easily reproduced using our codebase.

\begin{figure}[htp]
    \centering    
    \includegraphics[width=\linewidth]{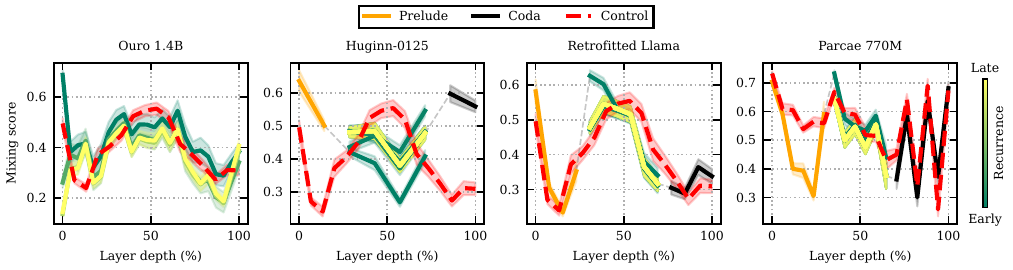}
    \caption{\textbf{Internal Mixing Stages per Recurrent Loop on the Hellaswag Dataset.}}
    \label{fig:summarised_recurrent_block_stages_hellaswag}
\end{figure}

\FloatBarrier
\subsection{Cyclic Functional Behavior}\label{appen:cyclic_functional}

In this section we provide further evidence for the cyclic ``functional'' behavior of looped LLM layers, including performance on downstream tasks and metrics capturing output token prediction behavior.

We begin by examining the multiple choice ``downstream'' behavior. Feedforward comparisons are shown in \Cref{fig:feedforward_multiple_choice}.

\begin{figure}[htp]
    \centering    
    \includegraphics[width=\linewidth]{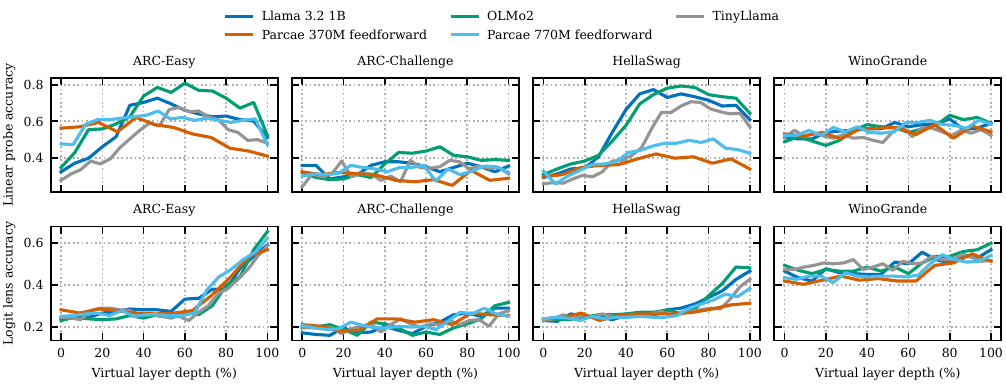}
    \caption{Linear probe and logitlens multiple choice performance per layer for a range of feedforward models.}
    \label{fig:feedforward_multiple_choice}
\end{figure}

We next present the linear probe and logitlens accuracy on the downstream multiple choice task for a wider range of datasets over all models in the main body of the paper: \ouro{} (\Cref{fig:downstream_ouro}), \raven{} (\Cref{fig:downstream_raven}), \rllama{} (\Cref{fig:downstream_rllama}), Parcae 770M (\Cref{fig:downstream_parcae_770}).

\begin{figure}[htp]
    \centering    
    \includegraphics[width=\linewidth]{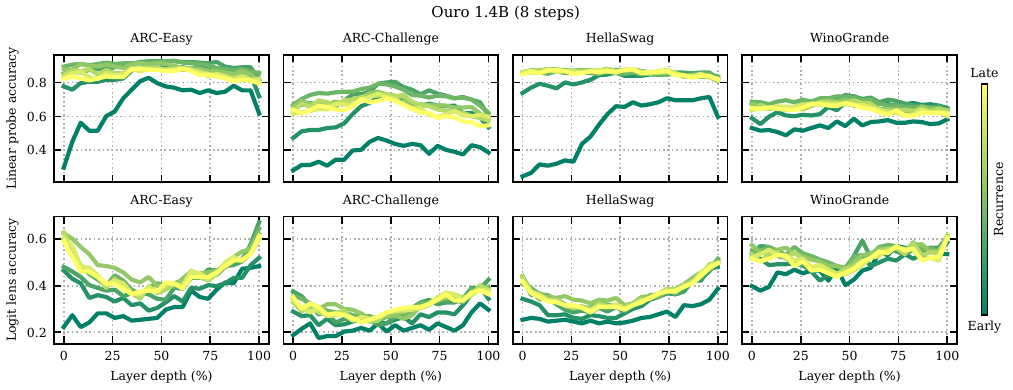}
    \caption{Linear probe and logitlens multiple choice performance per layer, per recurrence for \ouro{}}
    \label{fig:downstream_ouro}
\end{figure}

\begin{figure}[htp]
    \centering    
    \includegraphics[width=\linewidth]{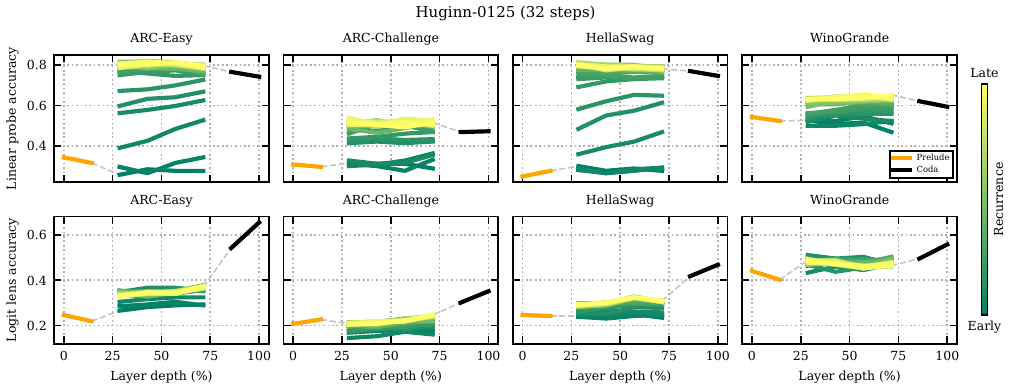}
    \caption{Linear probe and logitlens multiple choice performance per layer, per recurrence for \raven{}}
    \label{fig:downstream_raven}
\end{figure}

\begin{figure}[htp]
    \centering    
    \includegraphics[width=\linewidth]{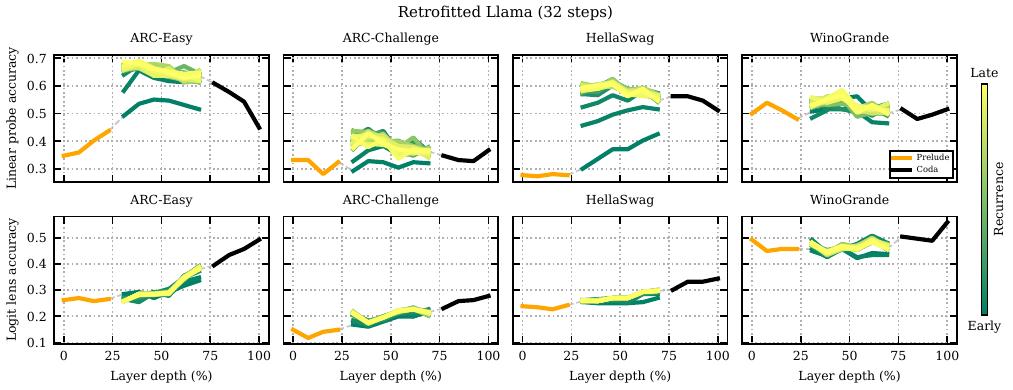}
    \caption{Linear probe and logitlens multiple choice performance per layer, per recurrence for \rllama{}}
    \label{fig:downstream_rllama}
\end{figure}

\begin{figure}[htp]
    \centering    
    \includegraphics[width=\linewidth]{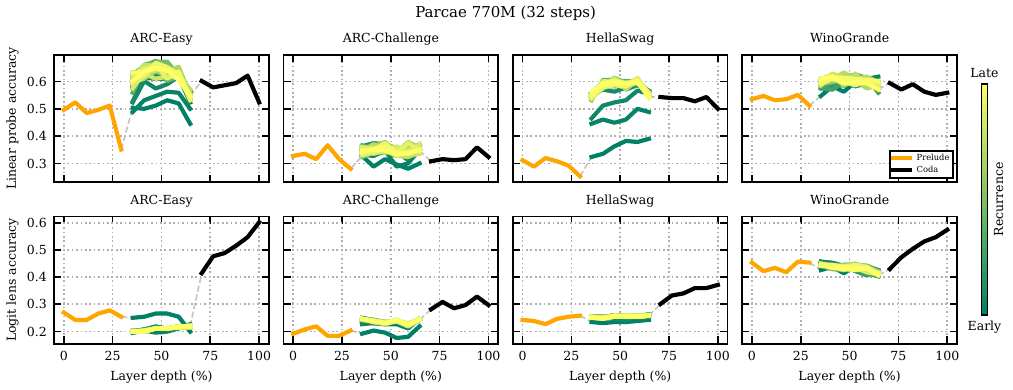}
    \caption{Linear probe and logitlens multiple choice performance per layer, per recurrence for Parcae 770M}
    \label{fig:downstream_parcae_770}
\end{figure}

\subsubsection{Do Looped LLMs Use Recurrences to Refine the Token Distribution?}\label{appen:entropylens}

We supplement this investigation into functional downstream behavior by plotting Entropy-Lens \citep{ali2025entropy}, investigating how these profiles compare to feedforward controls and whether looped LLMs use later recurrences to refine the output token distribution. We find in \Cref{fig:entropylens}, as elsewhere, that the looped models mirror their feedforward counterparts with each recurrence. We do not, however, observe any meaningful reduction in entropy across successive recurrences.

\begin{figure}[htp]
    \centering    
    \includegraphics[width=\linewidth]{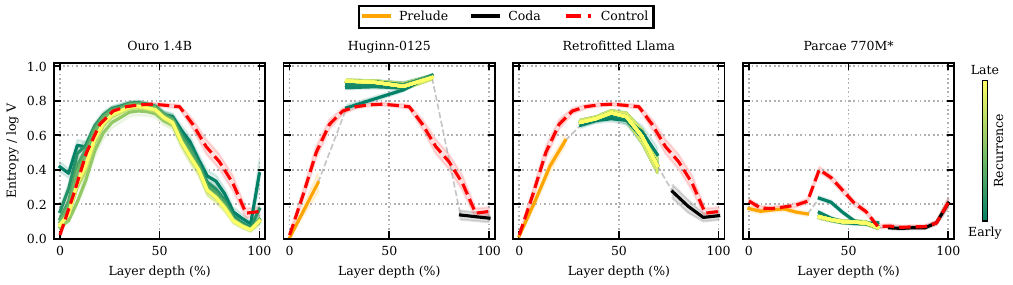}
    \caption{Entropy-lens plots for the main models investigated.}
    \label{fig:entropylens}
\end{figure}

\subsection{Self-Organization Into Stages of Inference}\label{appen:self_organising_soi}

One interesting question arising from our work is whether stages of inference emerge naturally or whether they are instead induced by specific aspects of the looped training procedure. The Parcae models provide strong evidence that they can emerge naturally, however these models are trained using a Poisson recurrence sampler and as such are trained ``in part'' as a feedforward model (although this makes up a minimal fraction of training examples).

To further isolate training method and provide additional evidence around the question of self-organisation, in this section we pretrain small-scale \acp{LT} from random initialization using a constant recurrence schedule ($T=4$) and a single loss term applied only to the final hidden state. This setup removes any explicit pressure for intermediate loops to produce meaningful or ``feedforward-like'' representations. Our training protocol is adapted from \citet{nanochat}; full experimental details are provided in \Cref{appen:experiment_details}.

Results for configurations with 4, 8, and 12 recurrent layers are visualized in \Cref{fig:summary_training_stages_of_interest}, alongside a depth-12 feedforward ``control''. Despite the absence of feedforward-biasing objectives, the models clearly \emph{self-organize} into distinct mixing stages within each recurrent block. While these results should be viewed with caution due to their scale, they provide evidence that the functional decomposition of inference into mixing and refinement phases may not purely be a training artifact. Rather, it may be a computationally advantageous strategy that emerges even under the constraints of weight sharing.

\begin{figure*}[ht]
    \centering
    \includegraphics[width=0.9\linewidth]{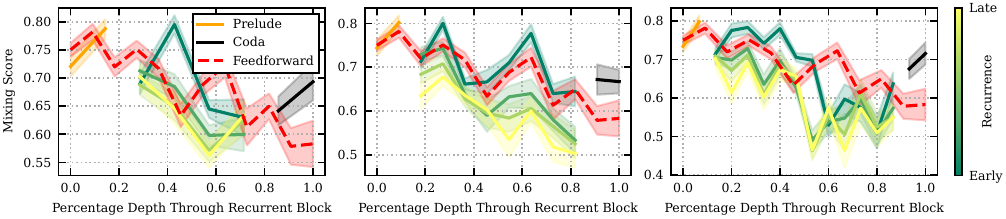}
    \caption{\textbf{Emergent Mixing  Stages in \acp{LT} Trained from Scratch.} Mixing Scores for models trained with a constant $T=4$ recurrence and loss applied only to the final state. Configurations follow $(2, L\otimes 4, 2)$ architecture for $L=4, 8, 12$ respectively (left to right). Red lines indicate a 12-layer feedforward control. Even without feedforward-biasing loss terms, the models self-organize into repeating functional cycles that mirror feedforward stages.}
    \vspace{-0.1cm}
    \label{fig:summary_training_stages_of_interest}
\end{figure*}

The extended stages of inference metrics for the models visualized in \Cref{fig:summary_training_stages_of_interest} are presented in \Cref{fig:r4s,fig:r8s,fig:r12s}. Equivalent models with added input injection are visualized in \Cref{fig:r4si,fig:r8si,fig:r12si}, and equivalent models without prelude/coda layers in \Cref{fig:r4,fig:r8,fig:r12}.

It appears in these small scale experiments that feedforward stages of inference are most closely replicated without input injection, and using prelude/coda layers, as in \Cref{fig:r4s,fig:r8s,fig:r12s}. The addition of input injection appears to mean that the \emph{final} recurrence follows the feedforward stages of inference more closely, but to the detriment of earlier recurrences.

\begin{figure}[ht]
    \centering
    \includegraphics[width=\linewidth]{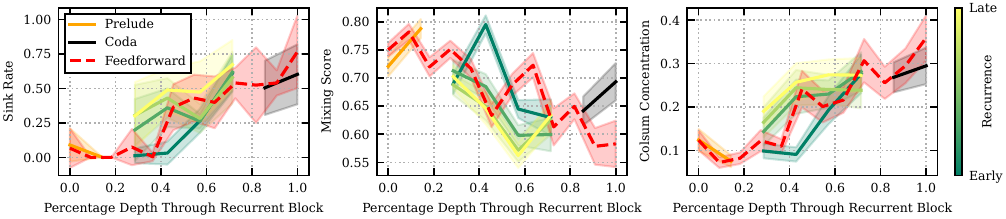}
    \caption{\textbf{Stages of inference metrics for a small-scale \ac{LT} of configuration $(2,4 \otimes 4,2)$}, compared to a ``control'' feedforward Transformer with the same training configuration and depth 12.}
    \label{fig:r4s}
\end{figure}

\begin{figure}[ht]
    \centering
    \includegraphics[width=\linewidth]{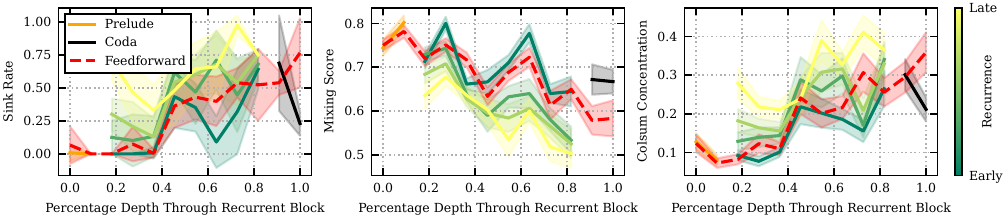}
    \caption{\textbf{Stages of inference metrics for a small-scale \ac{LT} of configuration $(2,8 \otimes 4,2)$}, compared to a ``control'' feedforward Transformer with the same training configuration and depth 12.}
    \label{fig:r8s}
\end{figure}

\begin{figure}[ht]
    \centering
    \includegraphics[width=\linewidth]{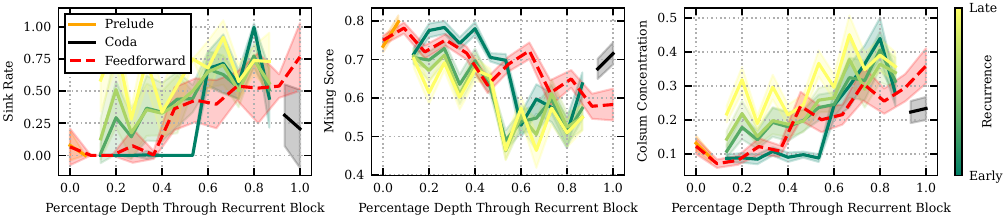}
    \caption{\textbf{Stages of inference metrics for a small-scale \ac{LT} of configuration $(2,12 \otimes 4,2)$}, compared to a ``control'' feedforward Transformer with the same training configuration and depth 12.}
    \label{fig:r12s}
\end{figure}

\begin{figure}[ht]
    \centering
    \includegraphics[width=\linewidth]{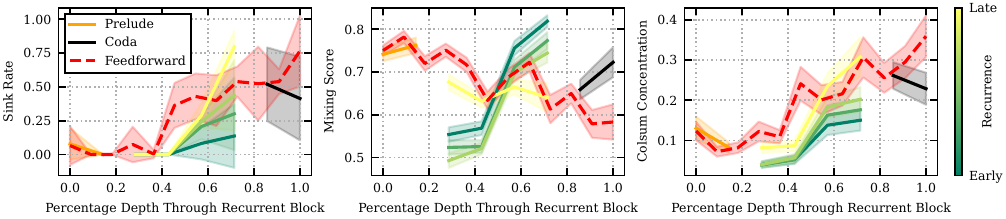}
    \caption{\textbf{Stages of inference metrics for a small-scale \ac{LT} of configuration $(2,4 \otimes 4,2)_I$}, compared to a ``control'' feedforward Transformer with the same training configuration and depth 12.}
    \label{fig:r4si}
\end{figure}

\begin{figure}[ht]
    \centering
    \includegraphics[width=\linewidth]{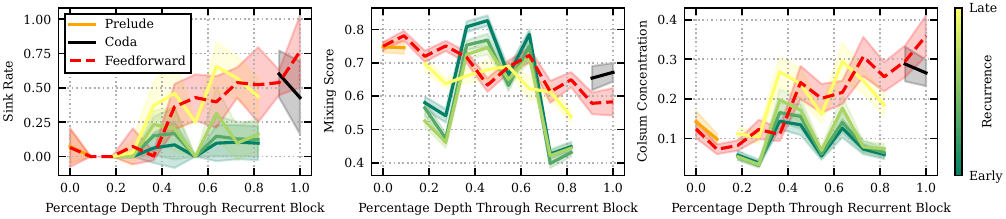}
    \caption{\textbf{Stages of inference metrics for a small-scale \ac{LT} of configuration $(2,8 \otimes 4,2)_I$}, compared to a ``control'' feedforward Transformer with the same training configuration and depth 12.}
    \label{fig:r8si}
\end{figure}

\begin{figure}[ht]
    \centering
    \includegraphics[width=\linewidth]{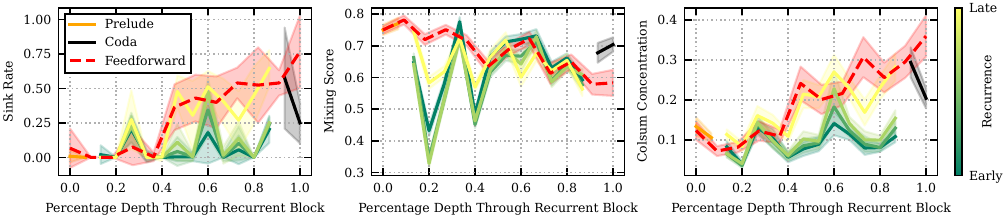}
    \caption{\textbf{Stages of inference metrics for a small-scale \ac{LT} of configuration $(2,12 \otimes 4,2)_I$}, compared to a ``control'' feedforward Transformer with the same training configuration and depth 12.}
    \label{fig:r12si}
\end{figure}

\begin{figure}[ht]
    \centering
    \includegraphics[width=\linewidth]{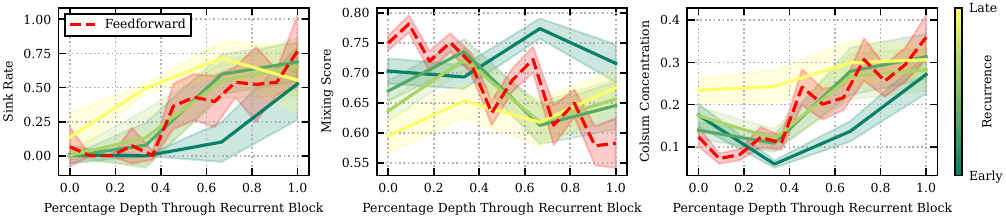}
    \caption{\textbf{Stages of inference metrics for a small-scale \ac{LT} of configuration $(0,4 \otimes 4,0)$}, compared to a ``control'' feedforward Transformer with the same training configuration and depth 12.}
    \label{fig:r4}
\end{figure}

\begin{figure}[ht]
    \centering
    \includegraphics[width=\linewidth]{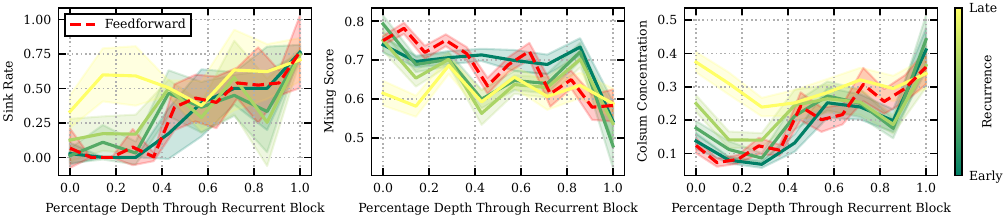}
    \caption{\textbf{Stages of inference metrics for a small-scale \ac{LT} of configuration $(0,8 \otimes 4,0)$}, compared to a ``control'' feedforward Transformer with the same training configuration and depth 12.}
    \label{fig:r8}
\end{figure}

\begin{figure}[ht]
    \centering
    \includegraphics[width=\linewidth]{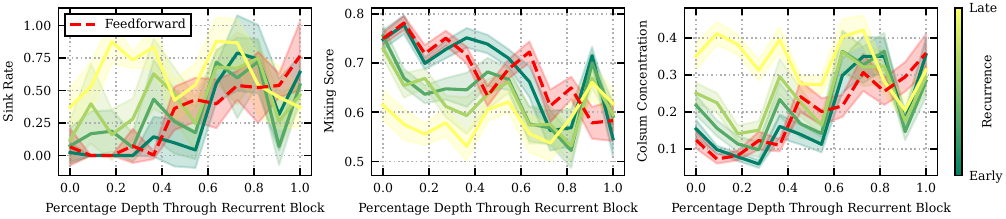}
    \caption{\textbf{Stages of inference metrics for a small-scale \ac{LT} of configuration $(0,12 \otimes 4,0)$}, compared to a ``control'' feedforward Transformer with the same training configuration and depth 12.}
    \label{fig:r12}
\end{figure}

%% file: sections/appendices/licenses.tex
\begin{table}[ht!]
  \centering
  \caption{Existing assets used in this work, with their licenses.}
  \label{tab:asset-licenses}
  \small
  \begin{tabular}{@{}>{\raggedright\arraybackslash}p{0.25\textwidth}
                    >{\raggedright\arraybackslash}p{0.47\textwidth}
                    >{\raggedright\arraybackslash}p{0.22\textwidth}@{}}
    \toprule
    Asset & Identifier / version & License \\
    \midrule
    \multicolumn{3}{@{}l}{\textit{Code}} \\
    LM Evaluation Harness & \nolinkurl{lm-eval} 0.4.13
      & \href{https://github.com/EleutherAI/lm-evaluation-harness/blob/main/LICENSE.md}{MIT} \\
    TransformerLens & \nolinkurl{transformer-lens} 3.9.0
      & \href{https://github.com/TransformerLensOrg/TransformerLens/blob/main/LICENSE}{MIT} \\
    nanochat & \nolinkurl{karpathy/nanochat}
      & \href{https://github.com/karpathy/nanochat/blob/master/LICENSE}{MIT} \\
    Parcae & \nolinkurl{sandyresearch/parcae}
      & \href{https://github.com/sandyresearch/parcae/blob/main/LICENSE}{MIT} \\
    Recurrent pretraining (Huginn) & \nolinkurl{seal-rg/recurrent-pretraining}
      & \href{https://github.com/seal-rg/recurrent-pretraining/blob/main/LICENSE}{Apache 2.0} \\
    Hugging Face Transformers & \nolinkurl{transformers} 5.17.0
      & \href{https://github.com/huggingface/transformers/blob/main/LICENSE}{Apache 2.0} \\
    \midrule
    \multicolumn{3}{@{}l}{\textit{Data}} \\
    GSM8K & \nolinkurl{openai/gsm8k} (\nolinkurl{main})
      & \href{https://github.com/openai/grade-school-math/blob/master/LICENSE}{MIT} \\
    HellaSwag & \nolinkurl{Rowan/hellaswag}
      & \href{https://huggingface.co/datasets/Rowan/hellaswag\#licensing-information}{MIT} \\
    ARC-Easy, ARC-Challenge & \nolinkurl{allenai/ai2_arc}
      & \href{https://huggingface.co/datasets/allenai/ai2_arc}{CC BY-SA 4.0} \\
    WinoGrande & \nolinkurl{allenai/winogrande} (\nolinkurl{winogrande_xl})
      & \href{https://github.com/allenai/winogrande\#license}{CC BY} \\
    FineWeb-Edu\textsuperscript{a} & \nolinkurl{SandyResearch/fineweb-edu-shuffled}
      & \href{https://huggingface.co/datasets/SandyResearch/fineweb-edu-shuffled}{ODC-By 1.0} \\
    \midrule
    \multicolumn{3}{@{}l}{\textit{Models}} \\
    Huginn-0125 & \nolinkurl{tomg-group-umd/huginn-0125}
      & \href{https://huggingface.co/tomg-group-umd/huginn-0125}{Apache 2.0} \\
    Ouro 1.4B & \nolinkurl{ByteDance/Ouro-1.4B}
      & \href{https://huggingface.co/ByteDance/Ouro-1.4B}{Apache 2.0} \\
    Ouro 2.6B & \nolinkurl{ByteDance/Ouro-2.6B}
      & \href{https://huggingface.co/ByteDance/Ouro-2.6B}{Apache 2.0} \\
    Retrofitted Llama\textsuperscript{b}
      & \nolinkurl{smcleish/Recurrent-Llama-3.2-train-recurrence-32}
      & \href{https://huggingface.co/smcleish/Recurrent-Llama-3.2-train-recurrence-32}{Apache 2.0} \\
    Retrofitted OLMo
      & \nolinkurl{smcleish/Recurrent-OLMo-2-0425-train-recurrence-32}
      & \href{https://huggingface.co/smcleish/Recurrent-OLMo-2-0425-train-recurrence-32}{Apache 2.0} \\
    Retrofitted TinyLlama
      & \nolinkurl{smcleish/Recurrent-TinyLlama-3T-train-recurrence-32}
      & \href{https://huggingface.co/smcleish/Recurrent-TinyLlama-3T-train-recurrence-32}{Apache 2.0} \\
    Parcae 140M, 370M, 770M\textsuperscript{c}
      & \nolinkurl{SandyResearch/parcae-140m}, \nolinkurl{-370m}, \nolinkurl{-770m}
      & \href{https://huggingface.co/SandyResearch/parcae-770m}{Not specified} \\
    Llama 3.2 1B & \nolinkurl{meta-llama/Llama-3.2-1B}
      & \href{https://github.com/meta-llama/llama-models/blob/main/models/llama3_2/LICENSE}{Llama 3.2 Community License} \\
    OLMo 2 1B & \nolinkurl{allenai/OLMo-2-0425-1B}
      & \href{https://huggingface.co/allenai/OLMo-2-0425-1B}{Apache 2.0} \\
    TinyLlama 1.1B
      & \nolinkurl{TinyLlama/TinyLlama-1.1B-intermediate-step-1431k-3T}
      & \href{https://huggingface.co/TinyLlama/TinyLlama-1.1B-intermediate-step-1431k-3T}{Apache 2.0} \\
    GPT-2 (124M) & \nolinkurl{openai-community/gpt2}
      & \href{https://huggingface.co/openai-community/gpt2}{MIT} \\
    \bottomrule
  \end{tabular}

  \vspace{0.5ex}
  \begin{minipage}{\textwidth}
    \footnotesize
    \textsuperscript{a}Pretraining data for the Parcae recurrence controls and
    stability-sweep models trained in this work; derived from
    \href{https://huggingface.co/datasets/HuggingFaceFW/fineweb-edu}{\nolinkurl{HuggingFaceFW/fineweb-edu}}
    (ODC-By 1.0, also subject to the Common Crawl terms of use).
    \textsuperscript{b}Fine-tuned from Llama 3.2 1B, so the
    \href{https://github.com/meta-llama/llama-models/blob/main/models/llama3_2/LICENSE}{Llama 3.2 Community License}
    also applies.
    \textsuperscript{c}The released weights and tokenizer state no license; the
    Parcae code is \href{https://github.com/sandyresearch/parcae/blob/main/LICENSE}{MIT}.
  \end{minipage}
\end{table}

%% file: checklist.tex
\section*{NeurIPS Paper Checklist}

\begin{enumerate}

\item {\bf Claims}
    \item[] Question: Do the main claims made in the abstract and introduction accurately reflect the paper's contributions and scope?
    \item[] Answer: \answerYes{} %
    \item[] Justification: We attempt to accurately convey the paper's contributions and scope in the abstract and introduction. We try to make this brief and explicitly clear with a ``key insights'' box at the end of the introduction.
    \item[] Guidelines:
    \begin{itemize}
        \item The answer \answerNA{} means that the abstract and introduction do not include the claims made in the paper.
        \item The abstract and/or introduction should clearly state the claims made, including the contributions made in the paper and important assumptions and limitations. A \answerNo{} or \answerNA{} answer to this question will not be perceived well by the reviewers. 
        \item The claims made should match theoretical and experimental results, and reflect how much the results can be expected to generalize to other settings. 
        \item It is fine to include aspirational goals as motivation as long as it is clear that these goals are not attained by the paper. 
    \end{itemize}

\item {\bf Limitations}
    \item[] Question: Does the paper discuss the limitations of the work performed by the authors?
    \item[] Answer: \answerYes{} %
    \item[] Justification: We highlight specific limitations throughout where appropriate (e.g. small scale of ``self-organization'' experiments), and discuss some high-level limitations in the conclusion.
    \item[] Guidelines:
    \begin{itemize}
        \item The answer \answerNA{} means that the paper has no limitation while the answer \answerNo{} means that the paper has limitations, but those are not discussed in the paper. 
        \item The authors are encouraged to create a separate ``Limitations'' section in their paper.
        \item The paper should point out any strong assumptions and how robust the results are to violations of these assumptions (e.g., independence assumptions, noiseless settings, model well-specification, asymptotic approximations only holding locally). The authors should reflect on how these assumptions might be violated in practice and what the implications would be.
        \item The authors should reflect on the scope of the claims made, e.g., if the approach was only tested on a few datasets or with a few runs. In general, empirical results often depend on implicit assumptions, which should be articulated.
        \item The authors should reflect on the factors that influence the performance of the approach. For example, a facial recognition algorithm may perform poorly when image resolution is low or images are taken in low lighting. Or a speech-to-text system might not be used reliably to provide closed captions for online lectures because it fails to handle technical jargon.
        \item The authors should discuss the computational efficiency of the proposed algorithms and how they scale with dataset size.
        \item If applicable, the authors should discuss possible limitations of their approach to address problems of privacy and fairness.
        \item While the authors might fear that complete honesty about limitations might be used by reviewers as grounds for rejection, a worse outcome might be that reviewers discover limitations that aren't acknowledged in the paper. The authors should use their best judgment and recognize that individual actions in favor of transparency play an important role in developing norms that preserve the integrity of the community. Reviewers will be specifically instructed to not penalize honesty concerning limitations.
    \end{itemize}

\item {\bf Theory assumptions and proofs}
    \item[] Question: For each theoretical result, does the paper provide the full set of assumptions and a complete (and correct) proof?
    \item[] Answer: \answerYes{} %
    \item[] Justification: To the best of our knowledge, we articulate all relevant assumptions in our proofs.
    \item[] Guidelines:
    \begin{itemize}
        \item The answer \answerNA{} means that the paper does not include theoretical results. 
        \item All the theorems, formulas, and proofs in the paper should be numbered and cross-referenced.
        \item All assumptions should be clearly stated or referenced in the statement of any theorems.
        \item The proofs can either appear in the main paper or the supplemental material, but if they appear in the supplemental material, the authors are encouraged to provide a short proof sketch to provide intuition. 
        \item Inversely, any informal proof provided in the core of the paper should be complemented by formal proofs provided in appendix or supplemental material.
        \item Theorems and Lemmas that the proof relies upon should be properly referenced. 
    \end{itemize}

    \item {\bf Experimental result reproducibility}
    \item[] Question: Does the paper fully disclose all the information needed to reproduce the main experimental results of the paper to the extent that it affects the main claims and/or conclusions of the paper (regardless of whether the code and data are provided or not)?
    \item[] Answer: \answerYes{} %
    \item[] Justification: This information is covered in some detail in the appendices.
    \item[] Guidelines:
    \begin{itemize}
        \item The answer \answerNA{} means that the paper does not include experiments.
        \item If the paper includes experiments, a \answerNo{} answer to this question will not be perceived well by the reviewers: Making the paper reproducible is important, regardless of whether the code and data are provided or not.
        \item If the contribution is a dataset and\slash or model, the authors should describe the steps taken to make their results reproducible or verifiable. 
        \item Depending on the contribution, reproducibility can be accomplished in various ways. For example, if the contribution is a novel architecture, describing the architecture fully might suffice, or if the contribution is a specific model and empirical evaluation, it may be necessary to either make it possible for others to replicate the model with the same dataset, or provide access to the model. In general. releasing code and data is often one good way to accomplish this, but reproducibility can also be provided via detailed instructions for how to replicate the results, access to a hosted model (e.g., in the case of a large language model), releasing of a model checkpoint, or other means that are appropriate to the research performed.
        \item While NeurIPS does not require releasing code, the conference does require all submissions to provide some reasonable avenue for reproducibility, which may depend on the nature of the contribution. For example
        \begin{enumerate}
            \item If the contribution is primarily a new algorithm, the paper should make it clear how to reproduce that algorithm.
            \item If the contribution is primarily a new model architecture, the paper should describe the architecture clearly and fully.
            \item If the contribution is a new model (e.g., a large language model), then there should either be a way to access this model for reproducing the results or a way to reproduce the model (e.g., with an open-source dataset or instructions for how to construct the dataset).
            \item We recognize that reproducibility may be tricky in some cases, in which case authors are welcome to describe the particular way they provide for reproducibility. In the case of closed-source models, it may be that access to the model is limited in some way (e.g., to registered users), but it should be possible for other researchers to have some path to reproducing or verifying the results.
        \end{enumerate}
    \end{itemize}

\item {\bf Open access to data and code}
    \item[] Question: Does the paper provide open access to the data and code, with sufficient instructions to faithfully reproduce the main experimental results, as described in supplemental material?
    \item[] Answer: \answerYes{} %
    \item[] Justification: Full code and documentation are provided. We have attempted to make reproduction extremely simple.
    \item[] Guidelines:
    \begin{itemize}
        \item The answer \answerNA{} means that paper does not include experiments requiring code.
        \item Please see the NeurIPS code and data submission guidelines (\url{https://neurips.cc/public/guides/CodeSubmissionPolicy}) for more details.
        \item While we encourage the release of code and data, we understand that this might not be possible, so \answerNo{} is an acceptable answer. Papers cannot be rejected simply for not including code, unless this is central to the contribution (e.g., for a new open-source benchmark).
        \item The instructions should contain the exact command and environment needed to run to reproduce the results. See the NeurIPS code and data submission guidelines (\url{https://neurips.cc/public/guides/CodeSubmissionPolicy}) for more details.
        \item The authors should provide instructions on data access and preparation, including how to access the raw data, preprocessed data, intermediate data, and generated data, etc.
        \item The authors should provide scripts to reproduce all experimental results for the new proposed method and baselines. If only a subset of experiments are reproducible, they should state which ones are omitted from the script and why.
        \item At submission time, to preserve anonymity, the authors should release anonymized versions (if applicable).
        \item Providing as much information as possible in supplemental material (appended to the paper) is recommended, but including URLs to data and code is permitted.
    \end{itemize}

\item {\bf Experimental setting/details}
    \item[] Question: Does the paper specify all the training and test details (e.g., data splits, hyperparameters, how they were chosen, type of optimizer) necessary to understand the results?
    \item[] Answer: \answerYes{} %
    \item[] Justification: The training side of this is only slightly applicable to our work via the Nanochat experiments: here we used the standard nanochat splits and training setup, as made clear in the paper. Our test data and splits are made clear in the appendices.
    \item[] Guidelines:
    \begin{itemize}
        \item The answer \answerNA{} means that the paper does not include experiments.
        \item The experimental setting should be presented in the core of the paper to a level of detail that is necessary to appreciate the results and make sense of them.
        \item The full details can be provided either with the code, in appendix, or as supplemental material.
    \end{itemize}

\item {\bf Experiment statistical significance}
    \item[] Question: Does the paper report error bars suitably and correctly defined or other appropriate information about the statistical significance of the experiments?
    \item[] Answer: \answerYes{} %
    \item[] Justification: We have included error bars in all appropriate plots.
    \item[] Guidelines:
    \begin{itemize}
        \item The answer \answerNA{} means that the paper does not include experiments.
        \item The authors should answer \answerYes{} if the results are accompanied by error bars, confidence intervals, or statistical significance tests, at least for the experiments that support the main claims of the paper.
        \item The factors of variability that the error bars are capturing should be clearly stated (for example, train/test split, initialization, random drawing of some parameter, or overall run with given experimental conditions).
        \item The method for calculating the error bars should be explained (closed form formula, call to a library function, bootstrap, etc.)
        \item The assumptions made should be given (e.g., Normally distributed errors).
        \item It should be clear whether the error bar is the standard deviation or the standard error of the mean.
        \item It is OK to report 1-sigma error bars, but one should state it. The authors should preferably report a 2-sigma error bar than state that they have a 96\% CI, if the hypothesis of Normality of errors is not verified.
        \item For asymmetric distributions, the authors should be careful not to show in tables or figures symmetric error bars that would yield results that are out of range (e.g., negative error rates).
        \item If error bars are reported in tables or plots, the authors should explain in the text how they were calculated and reference the corresponding figures or tables in the text.
    \end{itemize}

\item {\bf Experiments compute resources}
    \item[] Question: For each experiment, does the paper provide sufficient information on the computer resources (type of compute workers, memory, time of execution) needed to reproduce the experiments?
    \item[] Answer: \answerYes{} %
    \item[] Justification: Compute requirements are relatively small: all of the stages of inference experiments can be run in less than half a day on an A10 GPU. Nanochat training runs take a bit longer: up to 12 hours for one run on an H100. We specify this in the experimental details section of the appendix.
    \item[] Guidelines:
    \begin{itemize}
        \item The answer \answerNA{} means that the paper does not include experiments.
        \item The paper should indicate the type of compute workers CPU or GPU, internal cluster, or cloud provider, including relevant memory and storage.
        \item The paper should provide the amount of compute required for each of the individual experimental runs as well as estimate the total compute. 
        \item The paper should disclose whether the full research project required more compute than the experiments reported in the paper (e.g., preliminary or failed experiments that didn't make it into the paper). 
    \end{itemize}
    
\item {\bf Code of ethics}
    \item[] Question: Does the research conducted in the paper conform, in every respect, with the NeurIPS Code of Ethics \url{https://neurips.cc/public/EthicsGuidelines}?
    \item[] Answer: \answerYes{} %
    \item[] Justification: Our research complies with the code of ethics.
    \item[] Guidelines:
    \begin{itemize}
        \item The answer \answerNA{} means that the authors have not reviewed the NeurIPS Code of Ethics.
        \item If the authors answer \answerNo, they should explain the special circumstances that require a deviation from the Code of Ethics.
        \item The authors should make sure to preserve anonymity (e.g., if there is a special consideration due to laws or regulations in their jurisdiction).
    \end{itemize}

\item {\bf Broader impacts}
    \item[] Question: Does the paper discuss both potential positive societal impacts and negative societal impacts of the work performed?
    \item[] Answer: \answerYes{} %
    \item[] Justification: We believe the societal impacts of this particular work are negligible, but the field of recurrent depth for test time compute has clear potential negative impacts: we highlight these in the appendices.
    \item[] Guidelines:
    \begin{itemize}
        \item The answer \answerNA{} means that there is no societal impact of the work performed.
        \item If the authors answer \answerNA{} or \answerNo, they should explain why their work has no societal impact or why the paper does not address societal impact.
        \item Examples of negative societal impacts include potential malicious or unintended uses (e.g., disinformation, generating fake profiles, surveillance), fairness considerations (e.g., deployment of technologies that could make decisions that unfairly impact specific groups), privacy considerations, and security considerations.
        \item The conference expects that many papers will be foundational research and not tied to particular applications, let alone deployments. However, if there is a direct path to any negative applications, the authors should point it out. For example, it is legitimate to point out that an improvement in the quality of generative models could be used to generate Deepfakes for disinformation. On the other hand, it is not needed to point out that a generic algorithm for optimizing neural networks could enable people to train models that generate Deepfakes faster.
        \item The authors should consider possible harms that could arise when the technology is being used as intended and functioning correctly, harms that could arise when the technology is being used as intended but gives incorrect results, and harms following from (intentional or unintentional) misuse of the technology.
        \item If there are negative societal impacts, the authors could also discuss possible mitigation strategies (e.g., gated release of models, providing defenses in addition to attacks, mechanisms for monitoring misuse, mechanisms to monitor how a system learns from feedback over time, improving the efficiency and accessibility of ML).
    \end{itemize}
    
\item {\bf Safeguards}
    \item[] Question: Does the paper describe safeguards that have been put in place for responsible release of data or models that have a high risk for misuse (e.g., pre-trained language models, image generators, or scraped datasets)?
    \item[] Answer: \answerNA{} %
    \item[] Justification: We do not release any new models or data.
    \item[] Guidelines:
    \begin{itemize}
        \item The answer \answerNA{} means that the paper poses no such risks.
        \item Released models that have a high risk for misuse or dual-use should be released with necessary safeguards to allow for controlled use of the model, for example by requiring that users adhere to usage guidelines or restrictions to access the model or implementing safety filters. 
        \item Datasets that have been scraped from the Internet could pose safety risks. The authors should describe how they avoided releasing unsafe images.
        \item We recognize that providing effective safeguards is challenging, and many papers do not require this, but we encourage authors to take this into account and make a best faith effort.
    \end{itemize}

\item {\bf Licenses for existing assets}
    \item[] Question: Are the creators or original owners of assets (e.g., code, data, models), used in the paper, properly credited and are the license and terms of use explicitly mentioned and properly respected?
    \item[] Answer: \answerYes{} %
    \item[] Justification: We make it clear what work we build upon, licenses are explicitly mentioned in \Cref{appen:licenses}.
    \item[] Guidelines:
    \begin{itemize}
        \item The answer \answerNA{} means that the paper does not use existing assets.
        \item The authors should cite the original paper that produced the code package or dataset.
        \item The authors should state which version of the asset is used and, if possible, include a URL.
        \item The name of the license (e.g., CC-BY 4.0) should be included for each asset.
        \item For scraped data from a particular source (e.g., website), the copyright and terms of service of that source should be provided.
        \item If assets are released, the license, copyright information, and terms of use in the package should be provided. For popular datasets, \url{paperswithcode.com/datasets} has curated licenses for some datasets. Their licensing guide can help determine the license of a dataset.
        \item For existing datasets that are re-packaged, both the original license and the license of the derived asset (if it has changed) should be provided.
        \item If this information is not available online, the authors are encouraged to reach out to the asset's creators.
    \end{itemize}

\item {\bf New assets}
    \item[] Question: Are new assets introduced in the paper well documented and is the documentation provided alongside the assets?
    \item[] Answer: \answerNA{} %
    \item[] Justification: We do not introduce new assets in this paper.
    \item[] Guidelines:
    \begin{itemize}
        \item The answer \answerNA{} means that the paper does not release new assets.
        \item Researchers should communicate the details of the dataset\slash code\slash model as part of their submissions via structured templates. This includes details about training, license, limitations, etc. 
        \item The paper should discuss whether and how consent was obtained from people whose asset is used.
        \item At submission time, remember to anonymize your assets (if applicable). You can either create an anonymized URL or include an anonymized zip file.
    \end{itemize}

\item {\bf Crowdsourcing and research with human subjects}
    \item[] Question: For crowdsourcing experiments and research with human subjects, does the paper include the full text of instructions given to participants and screenshots, if applicable, as well as details about compensation (if any)? 
    \item[] Answer: \answerNA{} %
    \item[] Justification: We do not crowdsource or work with human subjects.
    \item[] Guidelines:
    \begin{itemize}
        \item The answer \answerNA{} means that the paper does not involve crowdsourcing nor research with human subjects.
        \item Including this information in the supplemental material is fine, but if the main contribution of the paper involves human subjects, then as much detail as possible should be included in the main paper. 
        \item According to the NeurIPS Code of Ethics, workers involved in data collection, curation, or other labor should be paid at least the minimum wage in the country of the data collector. 
    \end{itemize}

\item {\bf Institutional review board (IRB) approvals or equivalent for research with human subjects}
    \item[] Question: Does the paper describe potential risks incurred by study participants, whether such risks were disclosed to the subjects, and whether Institutional Review Board (IRB) approvals (or an equivalent approval/review based on the requirements of your country or institution) were obtained?
    \item[] Answer: \answerNA{} %
    \item[] Justification: We have no study participants.
    \item[] Guidelines:
    \begin{itemize}
        \item The answer \answerNA{} means that the paper does not involve crowdsourcing nor research with human subjects.
        \item Depending on the country in which research is conducted, IRB approval (or equivalent) may be required for any human subjects research. If you obtained IRB approval, you should clearly state this in the paper. 
        \item We recognize that the procedures for this may vary significantly between institutions and locations, and we expect authors to adhere to the NeurIPS Code of Ethics and the guidelines for their institution. 
        \item For initial submissions, do not include any information that would break anonymity (if applicable), such as the institution conducting the review.
    \end{itemize}

\item {\bf Declaration of LLM usage}
    \item[] Question: Does the paper describe the usage of LLMs if it is an important, original, or non-standard component of the core methods in this research? Note that if the LLM is used only for writing, editing, or formatting purposes and does \emph{not} impact the core methodology, scientific rigor, or originality of the research, declaration is not required.
    \item[] Answer: \answerNA{} %
    \item[] Justification: We do not use LLMs in the ways described here. LLMs were used to improve clarity of writing, and for code assistance: as per the Main Track Handbook this does not necessitate declaration.
    \item[] Guidelines:
    \begin{itemize}
        \item The answer \answerNA{} means that the core method development in this research does not involve LLMs as any important, original, or non-standard components.
        \item Please refer to our LLM policy in the NeurIPS handbook for what should or should not be described.
    \end{itemize}

\end{enumerate}